\documentclass{tlp}
\usepackage{ifthen}
\usepackage{url}
\usepackage{amssymb}
\usepackage{amsmath}
\usepackage{import}
\usepackage{xparse}
\usepackage{amsmath}
\usepackage[usenames,dvipsnames]{xcolor}
\usepackage{xspace}
\usepackage{datatool}
\usepackage{glossaries}

\providecommand\m[1]{\ensuremath{#1}\xspace}
\renewcommand{\m}[1]{\ensuremath{#1}\xspace}
\newcommand{\trval}[1]{\m{\mathbf{#1}}}

	\newcommand{\lrule}{\leftarrow}
	\newcommand{\cause}{\stackrel{c}{\lrule}}
	
	\newcommand{\ltrue}{\trval{t}}
	\newcommand{\lfalse}{\trval{f}}
	\newcommand{\lunkn}{\trval{u}}
	
	\newcommand{\Tr}{\ltrue}
	\newcommand{\Fa}{\lfalse}
	\newcommand{\Un}{\lunkn}

	\newcommand{\struct}{\m{I}}

	\newcommand{\rules}{\m{R}}
	
	\NewDocumentCommand\inter{g+g}{%
	  \IfNoValueTF{#1}
	    {\struct}
	    {\m{#1^{#2}}}}

	\renewcommand{\int}{\m{\mathbb{Z}}}

	\newcommand{\leqp}{\m{\leq_p}}

	\newcommand{\leqt}{\m{\leq_t}}
	\newcommand{\geqt}{\m{\geq_t}}
	
	\DeclareMathOperator\lfp{lfp}
	\DeclareMathOperator\gfp{gfp}

	\newcommand{\val}{\m{\nu}}

	\NewDocumentCommand\subs{g+g}{%
	  \IfNoValueTF{#1}
	    {\m{/}}
	    {\m{#1/ #2}}}

\newcommand{\ouracronym}[3]{%
	\newacronym{#1}{#2}{#3}
	\expandafter\newcommand\csname #1\endcsname{\gls{#1}\xspace}%
}
	\ouracronym{FO}{FO}{first-order logic}
	\ouracronym{PC}{PC}{propositional calculus}
	\ouracronym{MX}{MX}{Model Expansion}
	\ouracronym{MO}{MO}{Model Optimization}
	\ouracronym{ASP}{ASP}{Answer Set Programming}
	\ouracronym{TP}{TP}{Theorem Proving}
	\ouracronym{LP}{LP}{Logic Programming}
	\ouracronym{CP}{CP}{Constraint Programming}
	\ouracronym{FP}{FP}{Functional Programming}
	\ouracronym{KR}{KR}{Knowledge Representation}
	\ouracronym{CSP}{CSP}{Constraint Satisfaction Problem}
	\ouracronym{SMT}{SMT}{SAT Modulo Theories}
	\ouracronym{KBS}{KBS}{knowledge base system}
	\ouracronym{NNF}{NNF}{Negation Normal Form}
	\ouracronym{FNNF}{FNNF}{Flat Negation Normal Form}
	\ouracronym{DefNNF}{DefNNF}{Definition Negation Normal Form}
	\ouracronym{DEFNF}{DEFNF}{Definition Normal Form}
	\ouracronym{CDCL}{CDCL}{Conflict-Driven Clause-Learning}
	\ouracronym{WFS}{WFS}{Well-Founded Semantics}
	\ouracronym{LCG}{LCG}{Lazy Clause Generation}
	\ouracronym{AEL}{AEL}{Autoepistemic Logic}
	\ouracronym{OEL}{OEL}{Ordered Epistemic Logic}
	\ouracronym{AFT}{AFT}{Approximation Fixpoint Theory}

	\makeatletter
	\def\ifenv#1{
	\def\@tempa{#1}%
	\def\@ttempa{#1*}%
	\ifx\@tempa\@currenvir
	\expandafter\@firstoftwo
	\else
	\expandafter\@secondoftwo
	\fi
	}
	\makeatother

	\newcommand{\ddrule}[4]{\ensuremath{#1 \leftarrow #2 & \{#3\} & #4}}
	\newcommand{\drule}[2]{\ensuremath{#1 & \leftarrow & #2}}

	\newcommand{\darule}[4]{\ensuremath{#1 \leftarrow #2 & \{#3\} & #4}}
	\newcommand{\arule}[2]{\ensuremath{#1 \, &\leftarrow \, #2}}

	\newcommand{\LNDRule}[2]{
	\ifenv{array}
	{\drule{#1}{#2}}
	{ \ifenv{align}
		{\arule{#1}{#2}}
		{\ifenv{align*}
		{\arule{#1}{#2}}
		{ERROR: using LDRule in unsupported environment: \@currenvir}
		}
	}
	}

	\newcommand{\LDRule}[4]{
	\ifenv{array}
	{\ddrule{#1}{#2}{#3}{#4}}
	{ \ifenv{align}
		{\darule{#1}{#2}{#3}{#4}}
		{\ifenv{align*}
		{\darule{#1}{#2}{#3}{#4}}
		{ERROR: using LDRule in unsupported environment: \@currenvir}
		}
	}
	}

	\NewDocumentCommand\LRule{m+g+g+g}{%
		\IfNoValueTF{#2}%
		{#1.&}{%
		\IfNoValueTF{#3}
		{\LNDRule{#1}{#2.}}
		{\LDRule{#1}{#2.}{#3}{#4}}%
		}
	}

	\NewDocumentCommand\CLRule{m+g}{%
	\ifenv{array}
	{\cdrule{#1}{#2}}
	{ \ifenv{align}
		{\carule{#1}{#2}}
		{\ifenv{align*}
			{\carule{#1}{#2}}
			{ERROR: using CLRule in unsupported environment: \@currenvir}
		}
	}
	}

	\NewDocumentCommand\carule{m+g}{%
		\IfNoValueTF{#2}
			{\ensuremath{#1.}}
			{\ensuremath{#1 \, &\cause \, #2}}}
	\NewDocumentCommand\cdrule{m+g}{%
		\IfNoValueTF{#2}
			{\ensuremath{#1.}}
			{\ensuremath{#1 & \cause & #2}}}

	\newcommand{\algrule}[4]{
	\hbox{{#1}:}& 
	\quad #2 ~\longrightarrow~ #3 
	\hbox{~ if } #4\\
	}

	\newcommand{\AlgoRule}[4]{
	\ifenv{array}
	{\algrule{#1}{#2}{#3}{#4}}
		{ERROR: using AlgoRule in unsupported environment: \@currenvir}
	}

	\newcommand{\ignore}[1]{}

	\newboolean{nocomments}
	\setboolean{nocomments}{false}

	\newboolean{commentmargin}
	\setboolean{commentmargin}{true}

	\newcommand{\namedcomment}[3]{%
		\ifthenelse{\boolean{nocomments}}%
		{}
		{
			\ifthenelse{\boolean{commentmargin}}%
				{ {\color{#3} \marginpar{\color{#3}\sc #2}#1}  }
				{  {\color{#3} {\sc #2}: #1}  }
		}%
	}
	\newcommand{\mnamedcomment}[3]{\ifthenelse{\boolean{nocomments}}{}{{\marginpar{ \color{#3}{\sc #2}:#1}}}}

	\usepackage{soul}

\usepackage[normalem]{ulem} 
\makeatletter
\font\uwavefont=lasyb10 scaled 700
\def\spelling{\bgroup\markoverwith{\lower3.5\p@\hbox{\uwavefont\textcolor{Red}{\char58}}}\ULon}
\def\grammar{\bgroup\markoverwith{\lower3.5\p@\hbox{\uwavefont\textcolor{LimeGreen}{\char58}}}\ULon}
\def\phrasing{\bgroup\markoverwith{\lower3.5\p@\hbox{\uwavefont\textcolor{RoyalBlue}{\char58}}}\ULon}

\newcommand\remove{\bgroup\markoverwith{\textcolor{red}{\rule[0.5ex]{2pt}{0.4pt}}}\ULon}
\makeatother

\usepackage{etoolbox}

\newcommand\setcitation[2]{%
  \csdef{mycommoncitation#1}{#2}}

\setcitation{IDP}{WarrenBook/DeCatBBD14}
\setcitation{idp}{WarrenBook/DeCatBBD14}
\setcitation{fodot}{tocl/DeneckerT08}
\setcitation{foid}{tocl/DeneckerT08}
\setcitation{FOID}{tocl/DeneckerT08}
\setcitation{cplogic}{journal/tplp/VennekensDB10}
\setcitation{CPlogic}{journal/tplp/VennekensDB10}
\setcitation{CPLogic}{journal/tplp/VennekensDB10}
\setcitation{CP}{fai/Rossi06}
\setcitation{cp}{fai/Rossi06}
\setcitation{EZCSP}{lpnmr/Balduccini11}
\setcitation{KR}{Baral:2003}
\setcitation{ASPComp2}{lpnmr/DeneckerVBGT09}
\setcitation{ASPComp3}{journals/tplp/CalimeriIR14}
\setcitation{ASPComp4}{conf/lpnmr/AlvianoCCDDIKKOPPRRSSSWX13}
\setcitation{ASPComp5}{journals/ai/CalimeriGMR16}
\setcitation{ASPComp6}{jair/GebserMR17}
\setcitation{ASPComp7}{tplp/GebserMR20}
\setcitation{CPSupport}{ictai/DeCat13}
\setcitation{CPsupport}{ictai/DeCat13}
\setcitation{functionDetection}{iclp/DeCatB13}
\setcitation{FunctionDetection}{iclp/DeCatB13}
\setcitation{fodot2asp}{corr/DeneckerLTV19} 
\setcitation{Tarskian}{corr/DeneckerLTV19} 
\setcitation{TarskianSemanticsASP}{corr/DeneckerLTV19} 
\setcitation{Inca}{iclp/DrescherW12}
\setcitation{csp2asp}{ijcai/DrescherW11}
\setcitation{DPLLT}{cav/GanzingerHNOT04}
\setcitation{AspInPractice}{synthesis/2012Gebser}
\setcitation{ASPInPractice}{synthesis/2012Gebser}
\setcitation{clasp}{ai/GebserKS12}
\setcitation{oclingo}{kr/GebserGKOSS12}
\setcitation{clingo}{iclp/GebserKKOSW16}
\setcitation{gringo}{lpnmr/GebserST07}
\setcitation{cmodels}{aaai/GiunchigliaLM04}
\setcitation{inputster}{tplp/Jansen13}
\setcitation{DLV}{tocl/LeonePFEGPS06}
\setcitation{LearningPaper}{TPLP/BruynoogheBBDDJLRDV} 
\setcitation{clog}{iclp/BogaertsVDV14} 
\setcitation{foc}{iclp/BogaertsVDV14}
\setcitation{FOC}{iclp/BogaertsVDV14}
\setcitation{inferenceClog}{ecai/BogaertsVDV14}
\setcitation{examplesClog}{nmr/BogaertsVDV14b} 
\setcitation{AFT}{DeneckerMT00}
\setcitation{KBS}{iclp/DeneckerV08}
\setcitation{KBS-invitedtalk}{jelia/Denecker16}
\setcitation{KBPE}{inap/DePooterWD11}
\setcitation{lazyGrounding}{jair/CatDBS15} 
\setcitation{LazyGrounding}{jair/CatDBS15} 
\setcitation{lazygrounding}{jair/CatDBS15} 
\setcitation{lazygroundingASP}{ijcai/BogaertsW18} 
\setcitation{justifications}{lpnmr/DeneckerBS15} 
\setcitation{justificationsAlpha}{ijcai/BogaertsW18} 
\setcitation{ASP}{marek99stable}
\setcitation{satid}{sat/MarienWDB08}
\setcitation{lazyclausegeneration}{constraints/OhrimenkoSC09}
\setcitation{FP}{ACMCS/Hudak89}
\setcitation{GroundingWithBounds}{jair/WittocxMD10}
\setcitation{GroundWithBounds}{jair/WittocxMD10}
\setcitation{SAT}{faia/SilvaLM09}
\setcitation{HandbookOfSAT}{faia/2009-185}
\setcitation{LTC}{iclp/Bogaerts14}
\setcitation{SPSAT}{ictai/DevriendtBMDD12}
\setcitation{BreakID}{sat/DevriendtBBD16}
\setcitation{breakid}{sat/DevriendtBBD16}
\setcitation{LCG}{stuckeyLCG}
\setcitation{MiniZinc}{conf/cp/NethercoteSBBDT07}
\setcitation{minizinc}{conf/cp/NethercoteSBBDT07}
\setcitation{amadini}{cpaior/AmadiniGM13}
\setcitation{bootstrapping}{ngc/BogaertsJDJBD16}
\setcitation{Bootstrapping}{ngc/BogaertsJDJBD16}
\setcitation{GroundedFixpoints}{ai/BogaertsVD15}
\setcitation{PartialGroundedFixpoints}{ijcai/BogaertsVD15}
\setcitation{LogicBlox}{datalog/GreenAK12}
\setcitation{proB}{journals/sttt/LeuschelB08}
\setcitation{NaturalInductions}{KR/DeneckerV14} 
\setcitation{LP}{jacm/EmdenK76}
\setcitation{SMT}{faia/BarrettSST09}
\setcitation{AF}{ai/Dung95}
\setcitation{ADF}{kr/BrewkaW10}
\setcitation{af}{ai/Dung95}
\setcitation{adf}{kr/BrewkaW10}
\setcitation{ADFRevisited}{ijcai/BrewkaSEWW13}
\setcitation{adfrevisited}{ijcai/BrewkaSEWW13}
\setcitation{DefaultLogic}{ai/Reiter80}
\setcitation{DL}{ai/Reiter80}
\setcitation{AEL}{mo85}
\setcitation{minisat}{sat/EenS03}
\setcitation{completion}{adbt/Clark78}
\setcitation{ClarkCompletion}{adbt/Clark78}
\setcitation{wasp}{lpnmr/AlvianoDFLR13}
\setcitation{minisatid}{ictai/DeCat13}
\setcitation{lcg}{stuckeyLCG}
\setcitation{CEGAR}{jacm/ClarkeGJLV03}
\setcitation{cegar}{jacm/ClarkeGJLV03}
\setcitation{CuttingPlane}{or/DantzigFJ54}
\setcitation{kodkod}{tacas/TorlakJ07}
\setcitation{cdcl}{Marques-SilvaS99}
\setcitation{CDCL}{Marques-SilvaS99}
\setcitation{1UIP}{iccad/ZhangMMM01}
\setcitation{relevance}{ijcai/JansenBDJD16}
\setcitation{relevance-implementation}{aspocp/JansenBDJD16}
\setcitation{WFS}{GelderRS91}
\setcitation{wfs}{GelderRS91}
\setcitation{UnfoundedSet}{GelderRS91}
\setcitation{UFS}{GelderRS91}
\setcitation{stablesemantics}{iclp/GelfondL88}
\setcitation{StableSemantics}{iclp/GelfondL88}
\setcitation{shatter}{Shatter}
\setcitation{sbass}{drtiwa11a}
\setcitation{lparsemanual}{url:lparsemanual}
\setcitation{AIC}{ppdp/FlescaGZ04}
\setcitation{templates}{tplp/DassevilleHJD15}
\setcitation{templates2}{iclp/DassevilleHBJD16}
\setcitation{sat-to-sat}{aaai/JanhunenTT16}
\setcitation{sat-to-sat-qbf}{bnp/BogaertsJT16}
\setcitation{sat-to-sat-QBF}{bnp/BogaertsJT16}
\setcitation{sat-to-sat-SO}{kr/BogaertsJT16}
\setcitation{XSB}{SwiW12}
\setcitation{KCmap}{jair/DarwicheM02}
\setcitation{TLA}{DBLP:books/aw/Lamport2002}
\setcitation{EventB}{BookAbrial2010}
\setcitation{MX}{MitchellT05}
\setcitation{MIP}{Sierksma96}
\setcitation{perefectmodel}{minker88/Przymusinski88}
\setcitation{PerfectModel}{minker88/Przymusinski88}
\setcitation{perfectmodel}{minker88/Przymusinski88}
\setcitation{SafeInductions}{ijcai/BogaertsVD17}
\setcitation{AIC}{ppdp/FlescaGZ04}
\setcitation{aic}{ppdp/FlescaGZ04}
\setcitation{alpha}{lpnmr/Weinzierl17}
\setcitation{omiga}{jelia/Dao-TranEFWW12}
\setcitation{gasp}{fuin/PaluDPR09}
\setcitation{asperix}{lpnmr/LefevreN09a}
\setcitation{CTL}{lop/ClarkeE81}
\setcitation{AFT-AIC}{ai/BogaertsC18}
\setcitation{UltimateApproximator}{DeneckerMT04}
\setcitation{KripkeKleene}{Fitting85}
\setcitation{AFT-HO}{corr/CharalambidisRS18} 
\setcitation{HereThere}{Heyting30}
\setcitation{dAEL}{ijcai/HertumCBD16}
\setcitation{SDD}{ijcai/Darwiche11}
\setcitation{HEX}{ijcai/EiterIST05}
\setcitation{wADF}{aaai/BrewkaSWW18}
\setcitation{wADFfix}{corr/BrewkaSWW18}
\setcitation{TransitionSystems}{jacm/NieuwenhuisOT06}
\setcitation{galliwasp}{lopstr/MarpleG12}
\setcitation{GalliWasp}{lopstr/MarpleG12}
\setcitation{clingcon}{tplp/BanbaraKOS17}
\setcitation{lp2sat}{birthday/JanhunenN11}
\setcitation{lp2mip}{LIU12}
\setcitation{lp2diff}{lpnmr/JanhunenNS09}
\setcitation{lp2acyc}{ecai/GebserJR14}
\setcitation{PB}{faia/RousselM09}
\setcitation{pb}{faia/RousselM09}
\setcitation{CuttingPlanes}{dam/CookCT87}
\setcitation{RoundingSAT}{ijcai/ElffersN18}
\setcitation{PRS}{aaai/DixonG02}
\setcitation{sat4j}{jsat/BerreP10}
\setcitation{SAT4J}{jsat/BerreP10}
\setcitation{mingo}{LIU12}
\setcitation{pbmodels}{lpnmr/LiuT05}
\setcitation{HEF-LP}{lpnmr/GebserLL07}
\setcitation{HCF-LP}{amai/Ben-EliyahuD94}
\setcitation{aspcore2}{tplp/CalimeriFGIKKLM20}
\setcitation{AspCore2}{tplp/CalimeriFGIKKLM20}
\setcitation{SemanticWeb}{book/AntoniouGHH12}
\setcitation{maxsat}{faia/LiM09}
\setcitation{MaxSAT}{faia/LiM09}
\setcitation{MAXSAT}{faia/LiM09}
\setcitation{QBF}{faia/BuningB09}
\setcitation{qbf}{faia/BuningB09}
\setcitation{SMUS}{cp/IgnatievPLM15}
\setcitation{}{}
\setcitation{}{}
\setcitation{}{}
\setcitation{}{}
\setcitation{}{}
\setcitation{}{}
\setcitation{}{}
\setcitation{}{}
\setcitation{}{}
\setcitation{}{}
\setcitation{}{}
\setcitation{}{}
\setcitation{}{}
\setcitation{}{}
\setcitation{}{}
\setcitation{}{}
\setcitation{}{}
  
\usepackage[conf={restate,proof at the end,no link to proof}]{proof-at-the-end}
\usepackage[capitalize]{cleveref}

\usepackage{amsmath}
\usepackage{graphicx}
\usepackage{multirow}
\usepackage{mathtools}
\usepackage{amssymb}
\newtheorem{theorem}{Theorem}[section]
\newtheorem{example}[theorem]{Example}
\newtheorem{proposition}[theorem]{Proposition}
\newtheorem{definition}[theorem]{Definition}
\newtheorem{lemma}[theorem]{Lemma}
\newtheorem{corollary}[theorem]{Corollary}
\newtheorem{remark}[theorem]{Remark}
\usepackage{adjustbox}

\usepackage{mathrsfs}
\usepackage{tikz}
\usetikzlibrary{arrows}
\tikzset{vertex/.style = {shape=circle,draw,minimum size=1em}}
\tikzset{edge/.style = {->,> = latex'}}
\usepackage{tkz-berge}
\usepackage{tikz-cd}
\usetikzlibrary{trees,decorations.pathreplacing,calligraphy}
\usepackage{mdframed}

\crefname{conjecture}{Conjecture}{Conjectures}

\newcommand\setl[1]{\m{\left\lbrace #1 \right\rbrace}}
\newcommand\setprop[2]{\m{\left\lbrace #1 \; \middle| \; #2 \right\rbrace}}
\newcommand{\tild}{\m{\mathord{\sim}}}

\newcommand{\interp}{\m{\mathcal{I}}}
\newcommand{\be}{\m{\mathcal{B}}}
\newcommand{\besub}[1]{\m{\be_{\mathrm{#1}}}}
\newcommand{\bekk}{\m{\besub{KK}}}
\newcommand{\bewf}{\m{\besub{wf}}}
\newcommand{\best}{\m{\besub{st}}}
\newcommand{\besp}{\m{\besub{sp}}}
\newcommand{\becwf}{\m{\besub{cwf}}}
\newcommand{\becst}{\m{\besub{cst}}}
\newcommand{\F}{\m{\mathcal{F}}}
\newcommand{\jf}{\m{\mathcal{J}\kern-0.2em\mathcal{F}}}
\newcommand{\jfcomplete}[1][\rules]{\m{\left\langle \F, \Fd, #1\right\rangle}}
\newcommand{\js}{\m{\mathcal{J}\kern-0.2em\mathcal{S}}}
\newcommand{\jscomplete}[1][\be]{\m{\left\langle \F, \Fd, \rules, #1\right\rangle}}
\newcommand{\lf}{\m{\mathcal{L}}}

\newcommand{\Fp}{\m{\F_{+}}}
\newcommand{\Fn}{\m{\F_{-}}}
\newcommand{\Fd}{\m{\F_d}}
\newcommand{\Fo}{\m{\F_o}}
\newcommand{\pathstyle}[1]{\mathbf{#1}}
\newcommand{\branch}{\m{\pathstyle{b}}}
\newcommand{\branches}{\m{B}}

\newcommand{\opjf}{\m{O_{\jf}}}

\newcommand{\justifications}{\m{\mathfrak{J}}}
\newcommand{\edges}{\m{E}}

\newcommand{\suppoperator}[1]{\suppvalue\!\!_{#1}}

\newcommand\FOFD{\m{{\rm FO}({\rm FD})}}
\newcommand{\FOL}{\rm{FOL}}

\newcommand{\sel}{\mathcal{S}}

\newcommand{\folfactfull}[2]{\m{\mathfrak{#1}_{#2}}}
\newcommand{\folfact}[1]{\folfactfull{a}{#1}}

\DeclareMathOperator{\sgn}{sgn}

\DeclareMathOperator{\Def}{Def}
\DeclareMathOperator{\Open}{Open}

\DeclareMathOperator{\suppvalue}{SV}
\DeclareMathOperator{\jval}{val}

\DeclareMathOperator{\im}{Im}

\newcommand{\Fdl}{{\m{\F_{dl}}}}
\newcommand{\Fdlindex}[1]{\m{\F^{#1}_{dl}}}

\newcommand{\unfold}[3]{\UNFOLD_{\left(#1,#2\right)}\left(#3\right)}
\newcommand{\Unfold}[1]{\UNFOLD\left(#1\right)}
\newcommand{\Unf}[2]{\UNF_{#1}\left(#2\right)}

\newcommand{\nestedjs}[1][\js^1, \ldots, \js^k]{\left\langle \F, \Fd, \Fdl, \rules, \be, \setl{#1} \right\rangle}

\newcommand\restrict[2]{{%
    \left.\kern-\nulldelimiterspace%
    #1%
    \vphantom{|}%
    \right|_{#2}%
}}

\DeclareMathOperator{\Flat}{Flat}
\DeclareMathOperator{\UNFOLD}{Unfold}
\DeclareMathOperator{\UNF}{Unf}
\DeclareMathOperator{\Compress}{Compress}
\DeclareMathOperator{\Merge}{Merge}

\DeclareMathOperator{\shrink}{Shrink}
\DeclareMathOperator{\expand}{Expand}

\newcommand{\jframe}[1]{\left\{\begin{array}{l}#1\end{array}\right\}}
\newcommand{\leastframe}[1]{\left\lfloor\begin{array}{l}#1\end{array}\right\rfloor}
\newcommand{\greatestframe}[1]{\left\lceil\begin{array}{l}#1\end{array}\right\rceil}

\renewcommand\cite[1]{\citep{#1}}

\begin{document}

\lefttitle{Marynissen et al.}

\jnlPage{1}{32}
\jnlDoiYr{2021}
\doival{10.1017/xxxxx}

\title[On Nested Justification Systems]{On Nested Justification Systems  (full version) \thanks{This research was supported by the Flemish Government in the “Onderzoeksprogramma Artifici\"ele Intelligentie (AI) Vlaanderen” programme and by the FWO Flanders project G0B2221N.}} 

\begin{authgrp}
\author{\sn{Marynissen} \gn{Simon}\footnote{This work is part of the PhD thesis of the first author \cite{phd/Marynissen22}.
}}
\affiliation{KU Leuven}
\affiliation{Vrije Universiteit Brussel}
\author{\sn{Heyninck} \gn{Jesse}}
\affiliation{Vrije Universiteit Brussel}
\affiliation{University of Cape Town and CAIR, South-Africa}
\author{\sn{Bogaerts} \gn{Bart}}
\affiliation{Vrije Universiteit Brussel}
\author{\sn{Denecker} \gn{Marc}}
\affiliation{KU Leuven}
\end{authgrp}


\maketitle


\begin{abstract}
Justification theory is a general framework for the definition of semantics of rule-based languages that has a high explanatory potential. Nested justification systems, first introduced by  \citet{lpnmr/DeneckerBS15}, allow for the composition of justification systems. This notion of nesting thus enables the modular definition of semantics of rule-based languages, and increases the representational capacities of justification theory. As we show in this paper, the original characterization of semantics for nested justification systems leads to the loss of information relevant for explanations. In view of this problem, we provide an alternative characterization of their semantics and show that it is equivalent to the original one. 
  Furthermore, we show how nested justification systems allow representing fixpoint definitions. 
  
  This is an extended version of our paper that will be presented at ICLP 2022 and will appear in the special issue of TPLP with the ICLP proceedings. 

\end{abstract}



\section{Introduction}
\emph{Justification theory} \cite{lpnmr/DeneckerBS15} is a general framework for the definition of semantics of rule-based languages that allows to design languages with high explanatory potential, as the justification-based semantics give an immediate explanation of the truth-value of a fact.
In more detail, a justification is a tree of facts, where children of a given node occur in the body of a rule with this parent node as head. So-called \emph{branch evaluations} specify how to evaluate justifications.

Justification theory is not just useful for defining new logics, but also for \emph{unifying} existing ones: it
captures 
many different KR-languages, including logic programming and abstract argumentation, and their various semantics. By establishing a precise correspondence between semantics of different logics, justification theory sheds light on the common semantic mechanisms underlying these different logics.

\citet{lpnmr/DeneckerBS15} also introduced \emph{nested justification systems}, which allow justification systems to be composed by \emph{nesting} them. Essentially, a nested justification system can be seen as a tree of justification systems, where the subsystems provide sub-definitions for their parent systems. Such nested systems have several benefits. Firstly, they allow for the modular definition of (the semantics of) rule-based languages. For instance, the problem of defining a suitable semantics for logic programs with aggregates is notoriously difficult (witnessed, e.g., by the lack of consensus on semantics for logic programs with aggregates \cite{tplp/PelovDB07,faber2011semantics,gelfond2019vicious,corr/AlvianoFG21,VanbesienBD21}).
Nested justifications allow to separate distinct concerns in logic programs with aggregates:
on the one hand, we can define what is a justification for an aggregate expression and when it is ``acceptable'' (branch evaluation for aggregate expressions). On the other hand, we can define a branch evaluation (e.g., stable or 
well-founded) for rules.
        By nesting the justification system for aggregate expressions inside the justification system for programs, we then obtain a semantics for logic programs with aggregates without any additional effort required. Thus, the modular semantics of nested justification systems allow for the definition of complex KR-languages using a ``divide and conquer''-methodology. Secondly, as different branch evaluations can be used in different sub-systems, nested justification systems allow for the well-behaved combination of different semantics.
This can be useful when modelling the combination of knowledge from different agents that use a different semantics to interpret their respective knowledge bases.
Finally, nested justification systems allow to capture a richer class of logics than non-nested systems (e.g., fixpoint definitions, as we will show in this paper), and nesting thus increases the unifying power of justification theory.

\citet{lpnmr/DeneckerBS15} defined the semantics of nested justification systems by means of an operation called \emph{compression}, which turns an entire justification of the subsystem into a set of facts to be pasted into the supersystem.
However, properties of this characterisation of their semantics were never studied.
Furthermore, as we argue in this paper, the compression-based characterisation of these semantics lead to the loss of explanatory potential, as  information essential for explanations, such as the original rules, is lost.
Therefore, in this work,
 we give an alternative characterisation of the semantics of nested justification systems in terms of a so-called \emph{merging} operator.
Merging retains the original rules
 in the evaluation of a nested justification system, and therefore brings the explanatory potential of justification theory to nested systems.

 The contributions of this paper are as follows.
(1) An expanded exposition and semantic study of\textit{ nested justification systems} and the \textit{compression-based characterisation}  of their semantics. 
(2) The introduction of the \textit{merging-based characterisation} of their semantics. 
(3) A proof of \textit{equivalence} between compression and merging. 
(4) An application of nested justification systems to the representation of \textit{fixpoint definitions} \cite{tplp/HouDD10}.

The merging-based characterisation brings the explanatory potential of standard justification theory to \emph{nested} justification systems.
 This is promising not just for nested fixpoint definitions, which we study in the current paper, but also for other applications of nesting, such as the modular definition of new language constructs.
%
Furthermore, due to the general nature of justification theory, as well as of our characterisations and results, our results also apply to any future branch evaluations. 

 \noindent
 {\bf Outline of the paper}:
        The rest of this paper is structured as follows: necessary preliminaries on justification theory are given in Section \ref{sec:prelim}. In Section \ref{sec:just:systems}, nested justification systems are defined. The first characterisation of semantics for nested justification systems, compression-based characterisation, is recalled and studied in Section \ref{sec:compress}. In Section \ref{sec:merge}, the merging-based characterisation is introduced. These two characterisations are shown equivalent in Section \ref{sec:se}. Nested justification systems are shown to capture fixpoint definitions in Section \ref{sec:fo}. The paper is concluded in view of related work in Section \ref{sec:rel:work}.

\section{Preliminaries}\label{sec:prelim}

We use the formalisation of justification theory of \citet{tplp/MarynissenBD20}. We first give and explain all necessary definitions, and afterwards illustrate them with an example. 

In the rest of this paper, let $\F$ be a set, referred to as a \emph{fact space},
such that $\lf = \setl{\Tr, \Fa, \Un} \subseteq \F$,
where $\Tr$, $\Fa$ and $\Un$ have the respective meaning \emph{true}, \emph{false}, and \emph{unknown}.
The elements of $\F$ are called \emph{facts}.
The set $\lf$ behaves as the three-valued logic with truth order $\Fa \leqt \Un \leqt \Tr$.
We assume that $\F$ is equipped with an involution $\tild: \F \rightarrow \F$ (i.e., a bijection that is its own inverse)
such that $\tild \Tr=\Fa$, $\tild \Un=\Un$, and $\tild x \neq x$ for all $x \neq \Un$.
For any fact $x$, $\tild x$ is called the \emph{complement} of $x$.
An example of a fact space is the set of literals over a propositional vocabulary $\Sigma$ extended with $\lf$ where $\tild$ maps a literal to its negation.
For any set $A$ we define $\tild A$ to be the set of elements of the form $\tild a$ for $a \in A$.
We distinguish two types of facts: \emph{defined} and \emph{open} facts.
The former is accompanied by a set of rules that determine their truth value.
The truth value of the latter is not governed by the rule system but comes from an external source or is fixed (as is the case for logical facts), and only occur in bodies of rules.
\begin{definition}
  A \emph{justification frame} $\jf$ is a tuple $\jfcomplete$ such that
  \begin{itemize}
    \item $\Fd$ is a  subset of $\F\setminus\lf$ closed under $\tild$, i.e., $\tild \Fd = \Fd$; facts in $\Fd$ are called \emph{defined};\footnote{Thus, no logical fact is defined, or, equivalently, the logical facts are opens.}
    \item $\rules \subseteq \Fd \times 2^\F$;
    \item for each $x \in \Fd$, $(x, \emptyset) \notin R$ and there is an element $(x, A) \in R$ for $\emptyset \neq A \subseteq \F$.
  \end{itemize}
\end{definition}
The set of \emph{open} facts is denoted as $\Fo:=\F\setminus\Fd$.
An element $(x, A) \in \rules$ is called a \emph{rule} with \emph{head} $x$ and \emph{body} (or \emph{case}) $A$.
The set of cases of $x$ in $\jf$ is denoted as $\jf(x)$.
Rules $(x, A) \in \rules$ are denoted as $x \gets A$ and if $A=\setl{y_1, \ldots, y_n}$, we often write $x \gets y_1, \ldots, y_n$.

In justification theory, defined facts are evaluated by constructing \emph{justifications} for them. 
Justifications are directed graphs, where the set truth of the (labels of the) children of a node forms a reason (or argument, or cause, depending othe context) for the truth of the (label of the) node itself.
Such reasons are not necessarily convincing: for example if they are based on a parameter that is not true, or might lead to cyclic argumentation. Therefore, branches in justification trees are evaluated using \emph{branch evaluations}.

\begin{definition}
  Let $\jf=\jfcomplete$ be a justification frame.
  A \emph{justification} $J$ in $\jf$ is a (possibly infinite) directed labelled graph $(N, \Fd, \edges, \ell)$, where $N$ are the nodes, $E$ the vertices, and $\ell:N\to \Fd$ is a labelling function, such that (1) the underlying undirected graph is a forest,~i.e., is acyclic; and (2) for every internal node $n \in N$ it holds that $\ell(n) \gets \{\ell(m) \mid (n,m) \in \edges\} \in \rules$.
\end{definition}
We write $\justifications(x)$ for the set of justifications that have a node labelled $x$.   A justification is \emph{locally complete} if it has no leaves with label in $\Fd$.
  We call $x \in \Fd$ a \emph{root} of a justification $J$ if there is a node $n$ labelled $x$ such that every node is reachable from $n$ in $J$.
\begin{remark}
In some works, justifications are formalized as \emph{graphs} and the justifications as defined here are then called \emph{tree-like justifications} \cite{tplp/MarynissenBD20}. Since we restrict attention to the latter, we shall just use the term \emph{justification}. 
\end{remark}

\begin{definition}
  Let $\jf$ be a justification frame.
  A $\jf$-\emph{branch} is either an infinite sequence in $\Fd$ or a finite non-empty sequence in $\Fd$ followed by an element in $\Fo$.
  For a justification $J$ in $\jf$, a $J$-\emph{branch} starting in $x \in \Fd$ is a path in $J$ starting in $x$ that is either infinite or ends in a leaf of $J$.
  We write $\branches_J(x)$ to denote the set of $J$-branches starting in $x$.
\end{definition}
Not all $J$-branches are $\jf$-branches  since they can end in nodes with a defined fact as label.
However, if $J$ is locally complete, any $J$-branch is also a $\jf$-branch.

We denote a branch $\branch$ as $\branch:x_0 \rightarrow x_1 \rightarrow \cdots$ and define $\tild \branch$ as $\tild x_0 \rightarrow \tild x_1 \rightarrow \cdots$.

\begin{definition}
  A \emph{branch evaluation} $\be$ is a mapping that maps any $\jf$-branch to an element in $\F$ for all justification frames $\jf$.
  A justification frame $\jf$ together with a branch evaluation $\be$ form a \emph{justification system} $\js$, which is presented as a quadruple $\jscomplete$.
\end{definition}
A branch evaluation is \emph{parametric}  if every branch is mapped
to an open fact.
A justification system $\jscomplete$ is \textit{parametric} if $\be$ is parametric.

  The \emph{supported} (completion) branch evaluation $\besp$ maps $x_0 \rightarrow x_1 \rightarrow \cdots$ to $x_1$.
  The \emph{Kripke-Kleene} branch evaluation $\bekk$ maps finite branches to their last element and infinite branches to $\Un$.
%
  Let $\jf$ be a justification frame.
  A \emph{sign function} on $\jf$ is a map $\sgn: \Fd \rightarrow \setl{{-}, {+}}$ such that $\sgn(x) \neq \sgn(\tild x)$ for all $x \in \Fd$.
  We denote $\Fn\coloneqq \sgn^{-1}(\setl{{-}})$ and $\Fp\coloneqq \sgn^{-1}(\setl{{+}})$.
From now on, we fix a sign function on $\jf$. We say that an infinite branch has a positive (respectively negative) tail if from some point onwards all elements are in $\Fp$ (respectively $\Fn$).
%
  The \emph{well-founded} branch evaluation $\bewf$ maps finite branches to their last element. It maps infinite branches to $\Tr$ if they have a negative tail, to $\Fa$ if they have a positive tail and to $\Un$ otherwise.
  The \emph{co-well-founded branch evaluation} $\becwf$ maps finite
branches to their last element, infinite branches with a positive tail to $\Tr$, infinite branches with
a negative tail to $\Fa$, and all other infinite branches to $\Un$.
  The \emph{stable} (answer set) branch evaluation $\best$ maps a branch $x_0 \rightarrow x_1 \rightarrow \cdots$ to the first element that has a different sign than $x_0$ if it exists; otherwise $\branch$ is mapped to $\bewf(\branch)$.

The final ingredient of the semantics of justification systems are \emph{interpretations}, which are abstractions of possible states of affairs, formalized as an assignment of a truth value to each fact.
A \emph{(three-valued) interpretation} of $\F$ is a function $\interp:\F \rightarrow \lf$ such that $\interp(\tild x) = \tild \interp(x)$ and $\interp(l) = l$ for all $l  \in \lf$.
  The set of interpretations of $\F$ is denoted by $\mathfrak{I}_\F$.
We will call an interpretation \emph{two-valued} if $\interp(x)\neq\lunkn$ for all $x\neq\lunkn$.
%
Given an interpretation $\interp$ and a justification system $\js$, we can now 
evaluate the quality of justifications: the value assigned to a justification will be the least (in the truth order) value of its branches. The rationale behind this definition is that for a justification to be ``good'', all of the arguments it contains should be good as well. 
On top of this definition, we will define a notion of supported value of a fact, which is the value of its best justification. Indeed, in general we will not be interested in the existence of \emph{bad} arguments why a fact holds, but only in the existence of its best arguments.

  Let $\js=\jscomplete$ be a justification system, $\interp$ an interpretation of $\F$, and $J$ a locally complete justification in $\js$.
  Let $x \in \Fd$ be a label of a node in $J$.
  The \emph{value} of $x \in \Fd$ by $J$ under $I$ is defined as $\jval(J, x, \interp) = \min_{\branch \in \branches_J(x)} \interp(\be(\branch))$,
  where the minimum is taken with respect to $\leqt$.
  The \emph{supported value} of $x \in \F$ in $\js$ under $\interp$ is defined as
  \begin{equation*}
  \suppvalue(x, \interp) = \max_{J \in \justifications(x)} \jval(J,x,\interp)\text{  for $x \in \Fd$  \qquad $\suppvalue(x, \interp)=\interp(x)$ for $x \in \Fo$.
  }
  \end{equation*}

\begin{definition}
  Let $\js=\jscomplete$ be a justification system.
  An $\F$-interpretation $\interp$ is a $\js$-model if for all $x \in \Fd$, $\suppvalue_{\js}(x, \interp) = \interp(x)$.
  If $\js$ consists of $\jf$ and $\be$, then a $\js$-model can be referred to as a $\be$-model of $\jf$.
\end{definition}

%


\begin{definition}
  Two justification systems $\js_1$ and $\js_2$ are \emph{equivalent} if $\suppoperator{\js_1}=\suppoperator{\js_2}$.
\end{definition}



\begin{example}\label{ex:justification:system}
Consider the justification system $\jfcomplete$ with
$\Fd=\setl{p,\tild p,q,\tild q}$, $\F=\Fd\cup\{r,\tild r\}\cup{\cal L}$ and $\rules=\{p\gets \tild q,r; q\gets q; \tild p\gets q;\tild p\gets  \tild r; \tild q\gets \tild q\}$.


We have the following $\jf$-branches (we denote branches compactly in a graph-like fashion, i.e.\ a loop like $\branch_3$ denotes the infinite branch $q\rightarrow q\rightarrow \ldots$):

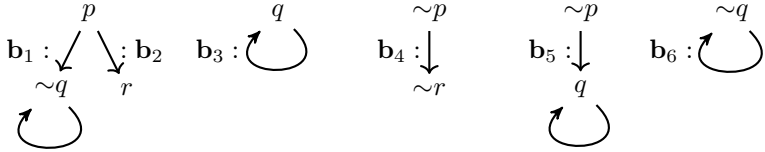
 \begin{center}
    \begin{tikzpicture}[transform shape]
        \node (l1) at (0.2,0.5) {$\branch_1:$};

        \node (p) at (1, 1) {$p$};
        \node (tq1) at (0.5, 0) {$\tild q$};
        \node (r) at (1.5, 0) {$r$};
        \node (l2) at (1.7,0.5) {$:\branch_2$};

        \node (l3) at (2.7,0.5) {$\branch_3:$};
        \node (q) at (3.5, 1) {$q$};

        \node (l4) at (5.1,0.5) {$\branch_4:$};
        \node (tp1) at (5.5, 1) {$\tild p$};
        \node (rp) at (5.5, 0) {$\tild r$};

        \node (l5) at (7.1,0.5) {$\branch_5:$};
        \node (tp2) at (7.5, 1) {$\tild p$};
        \node (qp) at (7.5, 0) {$q$};

        \node (l6) at (8.7,0.5) {$\branch_6:$};
        \node (tq) at (9.5, 1) {$\tild q$};

    \tikzstyle{EdgeStyle}=[style={->}]
    \Edge(p)(tq1)
    \Edge(p)(r)

	\Loop[dist=1.0cm,dir=SO](q)(q)

	    \Edge(tp1)(rp)

	    \Edge(tp2)(qp)
	    	\Loop[dist=1.0cm,dir=SO](qp)(qp)

	    	\Loop[dist=1.0cm,dir=SO](tq)(tq)
	    		\Loop[dist=1.0cm,dir=SO](tq1)(tq1)

    \end{tikzpicture}
  \end{center}
These branches are evaluated as follows:
%
%
%
\begin{center}
\def\arraystretch{0.8}
\begin{minipage}{0.46\textwidth}
 \begin{tabular}{@{\extracolsep{\fill}}lrrrrrr}
   $i$ & 1 & 2 & 3 & 4 & 5 & 6  \vspace{5pt}\midrule
$\besp(\branch_i)$ & $\tild q$ & $\phantom{\tild}r$ &$\phantom{\tild}q$ & $\tild r$ & $\phantom{\tild}q$ & $\tild q$\\
$\bewf(\branch_i)$& $\phantom{\tild}\Tr$ & $\phantom{\tild}r$ &$\phantom{\tild}\Fa$ & $\tild r$ & $\phantom{\tild}\Fa$ & $\phantom{\tild}\Tr$\\
$\becwf(\branch_i)$& $\phantom{\tild}\Fa$ & $\phantom{\tild}r$ &$\phantom{\tild}\Tr$ & $\tild r$ & $\phantom{\tild}\Tr$ & $\phantom{\tild}\Fa$
  \end{tabular}
\end{minipage}
\begin{minipage}{0.08\textwidth}
\end{minipage}
\begin{minipage}{0.46\textwidth}
  \begin{tabular}{@{\extracolsep{\fill}}lrrrrrr}
   $i$ & 1 & 2 & 3 & 4 & 5 & 6    \vspace{5pt}\midrule
 $\bekk(\branch_i)$& $\phantom{\tild}\Un$ & $\phantom{\tild}r$ &$\phantom{\tild}\Un$ & $\tild r$ & $\phantom{\tild}\Un$ & $\phantom{\tild}\Un$\\
$\best(\branch_i)$ & $\tild q$ & $\phantom{\tild}r$ &$\phantom{\tild}\Fa$ & $\tild r$ & $\phantom{\tild}q$ & $\phantom{\tild}\Tr$\\
$\phantom{x}$
  \end{tabular}
\end{minipage}
\end{center}

Let ${\cal I}$ be the interpretation with ${\cal I}(r)={\cal I}(p)={\cal I}(\tild q)=\Tr$, 
$J$ the justification for $p$ made up of $\branch_1$ and $\branch_2$,
and $\jf=\langle \F,\Fd,\rules,\bewf\rangle$.
We see that, e.g., $\jval(J, p, \interp) = \min_{\branch \in \branches_J(p)} \interp(\bewf(\branch))=\Tr$, and thus that $\suppvalue(p, \interp)= \interp(p)$.
In fact, it can be verified that this holds for every literal, i.e., $\interp$ is a $\jf$-model.
In $\interp$ (still, under the well-founded semantics),  $J$ serves as an explanation as to why $p$ is true: because $\tild q$ and $r$ are true.
Simultaneously, $J$ also explains why the facts $r$ and $\tild q$ are true. In the case of $r$, this is simply because it is an open fact (with value true). For $\tild q$, this explanation looks self-supporting: the reason why $\tild q$ holds is because $\tild q$ holds. For negative facts, stable and well-founded semantics accept such cyclic branches: the reason is that this cycle actually represents the fact that the positive fact ($q$) can only be justified by cyclic justifications, and well-founded and stable semantics reject cyclic justifications for positive facts.
\end{example}

\section{Nested Justification Systems}\label{sec:just:systems}
As outlined in the introduction, nested justifications, originally introduced by \citet{lpnmr/DeneckerBS15}, allow for the modular definition of semantics, the representation of richer classes of logics and the combination of different semantics in different modules. In this section, we introduce nested justification systems formally.

A nested justification system is essentially a tree structure of justification systems, meaning, that some systems are local to certain others,
i.e.\ some facts occurring as an open fact in one component system are defined in another component system.

\begin{definition}\label{def:nested}
  Let $\F$ be a fact space.
  A \emph{nested justification system} on $\F$ is a tuple $\langle \F, \Fd, \Fdl, \rules, \be, \setl{\js^1, \ldots, \js^k}\rangle$ such that:
  \begin{enumerate}
    \item $\langle \F, \Fdl, \rules, \be \rangle$ is a justification system;
    \item for each $i$, $\js^i$ is a nested justification system $\langle \F^i, \Fd^i, \Fdlindex{i}, \rules^i, \be^i, \ldots\rangle$; 
    \item $\Fd$ is partitioned into $\setl{\Fdl, \Fd^1, \ldots, \Fd^k}$;
    \item $\F = \cup_{i = 1}^k \F^k$;
    \item \label{def:openfactsofnestedsystem} $\Fo^i \subseteq \Fo \cup \Fdl$ (where $\Fo^{i} = \F^{i} \setminus \Fd^{i}$ as usual)

  \end{enumerate}
A nested justification system is called \emph{parametric} if \be is parametric and all of its subsystems are parametric. 
  We call a nested justification system \emph{compressible} if for each $i$, $\js^i$ is parametric.
\end{definition}
Given $\js=\langle \F, \Fd, \Fdl, \rules, \be, \setl{\js^1, \ldots, \js^k}$, we call (for $i=1,\ldots,k$) $\js$ the \textit{parent system of $\js_i$} and $\js_i$ a \textit{child system of $\js$}. 
Ancestor and descendant systems are defined analogously by transitively closing the parent, respectively, child relation. 
This defines a tree of nested justification systems, where the leaves have $\Fdl = \Fd$ and $k=0$, and thus correspond directly to an unnested justification system.

The factspace $\F$ consists of all the facts used in (some component system of) $\js$.
The facts in \Fdl\xspace are those that are defined \emph{locally} in the justification system, i.e., in the rules $R$. 
The facts in $\Fd$ are those that are either defined locally, or in some component system of $\js$. 
Every defined fact is either defined locally in the top system ($\Fdl$), or in one of the subsystems ($\Fd^i$).
Each child system $\js^i$ can use as opens only the opens of the root and the facts defined locally in the root.
This has the consequence that facts defined in $\js^i$ do not appear as opens in $\js^j$ if $i \neq j$.


\begin{lemma}\label{lemma:stratified:facts}
  Let $\js=\nestedjs$ be a nested justification system.
  If $i \neq j$, then $\Fd^i \cap \F^j = \emptyset$.
\end{lemma}


\begin{example}\label{running:example:1}
Let $\js= \nestedjs[\js^1]$ be a nested system with $\F = \setl{p,q,r,\tild p,\tild q,\tild r}\cup {\cal L}$, $\Fdl = \setl{r, \tild r}$, $\be=\bekk$, and $$\rules=
\jframe{
\begin{matrix*}[l]
r\gets p,q; &
\tild r\gets \tild p;& \tild r\gets \tild q\\
\end{matrix*}
}.$$
The inner system $\js^1$ is equal to the unnested system with
 $\Fdlindex{1} = \setl{p,\tild p,q,\tild q}$, $\be^1 = \bewf$, and
 $$\rules^1
= \jframe{
\begin{matrix*}[l]
p\gets \tild q,r ;&  \tild p\gets q;& \tild p \gets \tild r;&
q\gets  q;& \tild q\gets \tild q
\end{matrix*}
}.$$


We summarize this nested justification system graphically as follows:
  \begin{equation*}
\bekk:  \jframe{\begin{matrix*}[l]
r\gets p,q; &
\tild r\gets \tild p;& \tild r\gets \tild q\\
\end{matrix*}\\
\bewf
\jframe{
\begin{matrix*}[l]
p\gets \tild q,r ;&  \tild p\gets q;& \tild p \gets \tild r;&
q\gets  q;& \tild q\gets \tild q
\end{matrix*}

  }}.
  \end{equation*}
 This example also illustrates how different branch evaluations (e.g.,
 $\bekk$ and $\becwf$) can be combined in nested justification systems.
\end{example}

One notoriously difficult problem is the integration of aggregates in non-monotonic semantics of logic programming \cite{tplp/PelovDB07,faber2011semantics,gelfond2019vicious,corr/AlvianoFG21,VanbesienBD21}. 
While we do not develop the full theory here, we conjecture that nested justification systems can shed light on the relations between different semantics. 
As a first step towards this goal, we will show on a single example how different semantics for aggregates can be obtained by plugging in a different nested system defining the aggregate atoms in question. 
The outer system will always simply be evaluated under the stable semantics. 

\begin{example}\label{ex:aggregates:1}
We consider a representation of an aggregate ${\tt atLeastTwo}$ expressing that at least two atoms among $p$, $q$ and $s$ are true. 
%
%
%
\[
\small 
\js_{\mathit{FLP}} = 
\best :\jframe{
\begin{matrix*}[l]
p\gets\ltrue;& q\gets\ltrue;& s\gets p,{\tt atLeastTwo};
\end{matrix*} \\
\bekk:
\jframe{\begin{matrix*}[l]
{\tt atLeastTwo}\gets p,q;&
{\tt atLeastTwo}\gets s,q;&
{\tt atLeastTwo}\gets p,s\end{matrix*}
}
}
\]
%
\[
\small
\js_{\mathit{GZ}} = 
\best :\jframe{\begin{matrix*}[l]
p\gets\ltrue;& q \gets\ltrue;& s\gets p,{\tt atLeastTwo};
\end{matrix*} \\
\bekk:
\jframe{\begin{matrix*}[l]
{\tt atLeastTwo}\gets p,q,\tild s;&
{\tt atLeastTwo}\gets s,q,\tild p;\\
{\tt atLeastTwo}\gets p,s,\tild q;&
{\tt atLeastTwo}\gets p,q,s
\end{matrix*}
}
}
\]
For brevity and for emphasizing the relation with logic programming, we have here only written down the rules for \emph{positive} facts. The rules for negative facts can be obtained automatically by means of a technique called \emph{complementation}; see Appendix \ref{appendix:additional:preli} for details on this.
These systems differ only in their inner justification system, i.e., in which justifications for the atom ${\tt atLeastTwo}$ they accept. 
The inner system is evaluated here (in both cases) under Kripke-Kleene semantics; however, since it is a non-recursive system, all major branch evaluations we discussed coincide.
These two justification systems are inspired by the logic program 
\[p.\qquad q. \qquad s\lrule p,\  \#\{p, q, s\}\geq 2. \]
which has one (two-valued)  stable model, namely $\{p,q,s\}$, according to the FLP semantics \cite{faber2011semantics} but no (two-valued) stable models according to the GZ semantics \cite{gelfond2019vicious}. We will show later that this is indeed what the semantics for nested justifications for the first respectively the second system. 
\end{example}

In the following two sections, two different characterisations of the semantics for nested justification systems are given. In a nutshell, these characterisations both describe how to convert a nested justification system in a non-nested justification system. The first characterisation,  called the \emph{compression-based} characterisation, keeps the branch evaluation of the top system, but manipulates the rules to ``squeeze in'' all information about subsystems (their branch evaluation as well as their rules), while the second, \emph{merging-based} characterisation, keeps the rules of the original system intact, but creates a new branch evaluation based on all branch evaluations present in the system.



\section{Compression-based Characterisation of Semantics for Nested Systems}\label{sec:compress}
The semantics for nested justification systems was originally introduced by \citet{lpnmr/DeneckerBS15}. The basic idea is to \textit{compress} a nested justification system in a regular justification system, starting from the leaves of the nesting tree and iteratively moving up. In a \emph{compressible} nested system, the branch evaluations of subsystems are parametric and hence we know
(by Proposition 3 of \citet{ijcai/MarynissenBD21})
that, if the interpretation of the open facts is fixed, there is a unique model in which the value of a fact depends solely on the values of the opens. Therefore, the model can be represented by a set of rules for which the body contains only open facts. This representation is formed by transforming each justification into a single rule, by an operation called \emph{flattening}.
%

\begin{definition}
  Let $\js = \jscomplete$ be a 
  justification system.
  The \emph{flattening} $\Flat(\js)$ is the justification system $\langle \F, \Fd, \rules^f, \be \rangle$, where
  \begin{equation*}
  \rules^f = \setprop{x \gets A}{x \in \Fd, J \in \justifications_x, A = \setprop{\be(\branch)}{\branch \in \branches_J(x)}}.
  \end{equation*}
\end{definition}
Intuitively, $\Flat(\js)$ is obtained by constructing a rule $x\gets A$ for every justification $J$ for $x$, where $A$ is obtained by simply taking the facts to which the branches in $J$ starting in $x$ are mapped by the branch evaluation $\be$. 

\begin{example}[Example \ref{running:example:1} ctd.]\label{running:example:2}
In view of Example \ref{ex:justification:system}, the flattening of  the inner justification system $\js^1$ has the following rules:
$$(\rules^1)^f=
\jframe{
\begin{matrix*}[l]
p\gets \Tr,r;& \tild p\gets \Fa;& \tild p \gets \tild r;&
q\gets  \Fa;& \tild q\gets \Tr
\end{matrix*}
}$$
In more detail, $p\gets \Tr,r$ is obtained from the justification made up of $\branch_1$ and $\branch_2$, and the fact that $\bekk(\branch_1)=\Tr$ and $\bekk(\branch_2)=r$.

\end{example}

This is called flattening because every locally complete justification is reduced to a single rule. If $\be$ is parametric, justifications in this new system have  only branches of length two.  If these simple branches are mapped to their last element, then the flattening is equivalent to the original system.
For parametric systems, flattening thus preserves the meaning of the original system:

\begin{proposition}\label{prop:supportinvariantunderflat}
   If $\js_1$ and $\js_2$ are 
   parametric justification systems mapping a branch $x \rightarrow y$ to $y$, then:
\begin{enumerate}
\item $\js_1$ is equivalent to $\Flat(\js_1)$.
\item $\js_1$ and $\js_2$ are equivalent if and only if $\Flat(\js_1)$ and $\Flat(\js_2)$ are  equivalent.
\end{enumerate}
\end{proposition}
Notice that it immediately follows that if the preconditions of the above proposition are satisfied, then the $\js_1$ and $\Flat(\js_1)$-models coincide.
As a consequence, flattening preserves the properties of the justification at hand.
We should note, however, that the structure of the justification is lost: justifications in $\Flat(\js)$ are condensed down.
We will come back to this later.

Flattening provides us with a way to evaluate inner definitions. The second operation needed for compression is \emph{unfolding}, which allows to eliminate inner symbols from the outside definitions, by replacing facts at a lower level by a (flattened) case for that fact.

\begin{definition}
  Let $\rules$ be a set of rules, and $\rules_\ell$ a set of rules for the facts $X$ (the elements of $X$ are the heads of the rules in $\rules_\ell$).
  Take a rule $x \gets A$ in $\rules$.
  Let $f$ be any function with domain $A \cap X$ such that for all $y \in A \cap X$, $y \gets f(y)$ is a rule in $\rules_\ell$ (the function $f$ chooses a rule for each $y \in A \cap X$).
  The \emph{unfolding} of $x \gets A$ with respect to $f$ is the rule
  \begin{equation*}
  \Unf{f}{x \gets A} =
  x \gets (A \setminus X) \cup \bigcup_{y \in A \cap X} f(y).
  \end{equation*}
  Let $F_{x \gets A}$ be the set of such functions $f$.
  Then the \emph{unfolding} of $x \gets A$ with respect to $\rules_\ell$ is the set
  \begin{equation*}
  \Unf{\rules_\ell}{x \gets A}=\setprop{\Unf{f}{x \gets A}}{f \in F_{x \gets A}}.
  \end{equation*}
  The \emph{unfolding} of $\rules$ with respect to $\rules_\ell$ is
  \begin{equation*}
  \unfold{X}{\rules_\ell}{\rules} =
  \bigcup_{x \gets A \in \rules} \Unf{\rules_\ell}{x \gets A}.
  \end{equation*}
\end{definition}

In a nutshell, $  \unfold{X}{\rules_\ell}{\rules}$ is obtained by replacing, in each $x\gets y_1,\ldots,y_n\in\rules$, each fact $y_i$ defined in $\rules_\ell$ by the body facts $A$ of a rule $y_i\gets A$ defining $y_i$.

\begin{example}[Examples \ref{running:example:1} \& \ref{running:example:2} ctd.]\label{running:example:3}
With all systems as defined previously, we see that
\[\unfold{\Fdl}{\Flat(\rules^1)}{\rules}=
\jframe{
r\gets \Tr,r,\Fa;\quad
\tild r\gets \tild r;\quad \tild r\gets \Fa;\quad \tild r\gets \Tr
} \cup \rules^1.\]

\end{example}

On the basis of the unfolding of a set of rules we can define the unfolding of a justification system, obtained by  unfolding the rules of the parent system with respect to the rules of the children systems, and keeping the rules for the lower-level facts ($\rules^s$):
\begin{definition}
  Let $\js=\nestedjs$ be a two-level nested compressible justification system.
  The \emph{unfolding} of $\js$ is
  \begin{equation*}
  \Unfold{\js} = \langle \F, \Fd, \rules^s \cup \unfold{\Fd \setminus \Fdl}{\rules^s}{\rules}, \be \rangle,
  \end{equation*}
  where $\rules^s = \cup_{i=1}^k \rules^i$.
\end{definition}
%
The unfold-operation reduces the depth of the nesting tree and thus serves as a basis for the compression of a nested justification system (notice that since the subsystems $\js^i$ will have a smaller depth than $\js$, $\Compress(\js)$ is well-defined).

\begin{definition}[\citeauthor{lpnmr/DeneckerBS15} \citeyear{lpnmr/DeneckerBS15}]
  Let $\js=\nestedjs$ be a compressible nested justification system.
  The compression $\Compress(\js)$ is defined inductively to be the justification system
  \begin{equation*}
  \Unfold{\langle \F, \Fd, \Fdl, \rules, \be, \setl{\Flat(\Compress(\js^1)), \ldots, \Flat(\Compress(\js^k))}\rangle}.
  \end{equation*}
\end{definition}

\begin{remark}\label{remark:compress:and:parametric}
Intuitively, the compression works as follows, described here for a two-level system $\js$ with only one subsystem $\js^1=\langle \F^1,\Fd^1,\rules^1,\be^1  \rangle$.
First of all, $\js^1$ is flattened.
We know from Proposition \ref{prop:supportinvariantunderflat} that this preserves the meaning of the system, and in fact, it no longer matters which branch evaluation is used, provided that it maps branches of length 2 onto their last element. In other words, the original system and branch evaluation are consolidated into the flattened frame.
The next step is to replace each occurrence (in $\js$) of each fact defined in $\js^1$ by all possible cases for it in the flattened system.
If $\be^1$
is parametric, each such case in the flattened system contains only fact that are either locally defined or open in $\js$. The result can then be evaluated  without any knowledge of the branch evaluation or rules of $\js^1$.
While technically, the compressed system also contains a copy of $\rules^1$, they are only there for recovering the value of facts in $\Fd^1$; they play no role for determining the value of facts in $\js$.
In case $\be^1$ is not parametric, this construction no longer works: the cases in the flattening can then contain facts that are locally defined in $\jf^1$.
For this reason, compression was only defined for parametric branch evaluations by \citet{lpnmr/DeneckerBS15}.
We here slightly relax this condition, by allowing (in a compressible system) the outer branch evaluation to be non-parametric.

\end{remark}

\begin{example}[Example \ref{running:example:3} ctd.]
We have already calculated $\unfold{\Fdlindex{1}}{\rules^2}{\rules^1}$ in Example \ref{running:example:3}.
Since $\unfold{\Fdl}{\rules^1}{\rules}$ contains only rule bodies with open facts, we see that $\Flat(\unfold{\Fdl}{\rules^1}{\rules})=\unfold{\Fdl}{\rules^1}{\rules}$.
We therefore obtain:
$$\Compress(\js)=\langle \F,\Fd,\unfold{\Fdl}{\rules^1}{\rules},\bekk\rangle$$
and thus end up with an unnested justification system. It can be verified that the unique $\Compress(\js)$-model is $\interp$ with $\interp(\tild r)=\interp(p)=\interp(q)=\Fa$.
\end{example}

\begin{example}
We return to the nested justification system of Example \ref{ex:aggregates:1}. It can be verified that  $\Compress(\js_{\mathit{FLP}})$ is equal to (the complementation of) 
\[
\best:\jframe{
\begin{matrix*}[l]
\begin{matrix*}[l]
p\gets\ltrue;& q\gets\ltrue;& s\gets p,s& s\gets p,q,s;& s\gets p,q; 
\end{matrix*}
\\
\begin{matrix*}[l]
{\tt atLeastTwo} \gets p,q; & {\tt atLeastTwo} \gets s,q;  & {\tt atLeastTwo} \gets p,s; \end{matrix*}

\end{matrix*}
}
\]
The interpretation $\interp$ with ${\cal I}(s)={\cal I}(p)={\cal I}(q)=\Tr$ is the only two-valued  $\Compress(\js_{\mathit{FLP}})$. This interpretation indeed corresponds to the only stable model under the FLP semantics (mentioned above). 

On the other hand, there are no two-valued $\Compress(\js_{\mathit{GZ}})$-models. This can be seen by observing that 
$\Compress(\js_{\mathit{GZ}})$ equals (the complementation of):
\[
\best:\jframe{
\begin{matrix*}[l]
p\gets\ltrue;& q\gets\ltrue;& s\gets p,q,\tild s;& s\gets p,s,\tild q;& s\gets p, q,s,\tild p; & s\gets p,q,s;\\
\end{matrix*}\\
\begin{matrix*}[l]
{\tt atLeastTwo}\gets p,q,\tild s;&
{\tt atLeastTwo}\gets s,q,\tild p;\\
{\tt atLeastTwo}\gets p,s,\tild q;&
{\tt atLeastTwo}\gets p,q,s
\end{matrix*}
}
\]
Suppose $\interp$ were a $\Compress(\js_{\mathit{GZ}})$-model, then clearly $\interp(p)=\interp(q)=\ltrue$.
If $\interp(s)=\ltrue$, then $\suppvalue(s,\interp)=\lfalse$ (because each justification of $s$ has a branch $s\to \tild x \to \dots$ (with $x$ either $p$, $q$, or $s$). 
On the other hand, if $\interp(s)=\lfalse$, then  $\suppvalue(s,\interp)=\ltrue$, and it can be verified that the justification that selects the rule $s\gets p,q,\lnot s$ supports $s$ in this case. Hence, such an $\interp$ can indeed not exist. 
%
\end{example}


In this section, the semantics of nested justifications in terms of compressions, originally proposed by \citet{lpnmr/DeneckerBS15}  was described and studied. This characterisation of semantics of nested justification systems is the first one given in the literature. However, it does have a significant downside: the explanatory potential of justification systems is partially lost, since rules in lower levels of the justification system are compressed to their evaluation. For example, the definition of $r$ in Example  \ref{running:example:1} is compressed to $\{r\gets \Tr,r,\Fa\}$, and thus the explanatory potential of justification theory is completely lost.
Indeed,
intuitively, we have a justification in the original nested justification system that looks like the justification on the left of the figure below, and is obtained by constructing a justification on the basis of all (relevant) rules in $\js$:

%
%

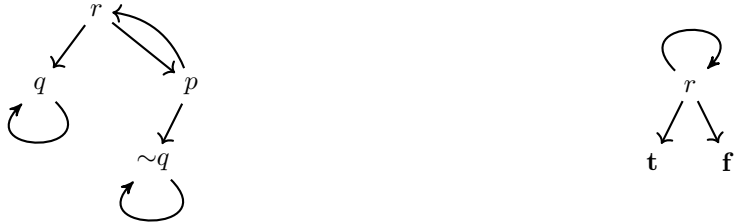
\begin{center}
    \begin{tikzpicture}[transform shape]
            \node (rt) at (-0.25,2) {$r$};
        \node (q) at (-1,1) {$q$};

        \node (p) at (1, 1) {$p$};
        \node (tq1) at (0.5, 0) {$\tild q$};

    \tikzstyle{EdgeStyle}=[style={->}]
	    \Edge(rt)(q)
	    \Edge(rt)(p)
	    \Edge(p)(tq1)

         \node (r2) at (7.6,1) {$r$};
        \node (Tr) at (7.1, 0) {$\Tr$};
        \node (Fa) at (8.1, 0) {$\Fa$};
%
%

    \tikzstyle{EdgeStyle}=[style={->}]
   
	   	   \Loop[dist=1.0cm,dir=SO](q)(q)
	   	   \Loop[dist=1.0cm,dir=SO](tq1)(tq1)

	      \Edge(r2)(Tr)
	    \Edge(r2)(Fa)
	        \Loop[dist=1.0cm,dir=NO](r2)(r2)
	       \tikzstyle{EdgeStyle}=[style={->, bend right}]
	    \Edge(p)(rt)
    \end{tikzpicture}
  \end{center}
However, under the compression-based semantics, we obtain the justification on the right, where most information on the justification of $r$ has been lost. Thus, from the point of view of explainable reasoning, it is preferable to be able to somehow evaluate the justification on the left while retaining semantic equivalence with the compression-based semantics. This is exactly what will be done in the next section.


\section{Merging-based Characterisation of Semantics for Nested Systems}\label{sec:merge}

To avoid the loss of information with explanatory potential suffered by the compression-based characterisation, the merging-based characterisation of semantics is now proposed. The basic idea of the merging-based characterisation is to
consolidate all component systems of
the nested justification system in a justification system $\langle \F, \Fd, \rules^*, \be^* \rangle$, where all rules of all component justification systems are simply gathered in $\rules^*$. This also means that justification branches are now constructed on the basis of \textit{all} rules occurring somewhere in the nested justification system, and not simply on the basis of one component justification system.
The information about the nesting structure of the original system is not discarded, but taken into account in the \emph{merge branch evaluation} $\be^*$. In more detail, for infinite branches $\branch$ in $\rules^*$, one looks for the highest component justification system $\js'$ in the nested justification system s.t.\ $\branch$ contains infinitely many elements of locally defined facts of $\js'$, and then applies the original branch evaluation $\be'$ of the component system in question to the $\js'$-branch $\branch'$ that corresponds to $\branch$. 
This formalizes the intuition that priority is given to the outermost semantics.

\begin{definition}\label{def:merge}
  The \emph{merge} $\Merge(\js)$ of $\js$ is the justification system $\langle \F, \Fd, \rules^*, \be^* \rangle$, where:
\begin{itemize}
\item    $\rules^*$ is the union of all the rules in $\js$.
\item  the \emph{merge branch evaluation} $\be^*$
w.r.t.\ $\js$
is defined as follows:
 \begin{enumerate}
    \item if $\branch$ is finite, then $\be^*(\branch)$ is equal to the last element of $\branch$.
    \item if $\branch$ is infinite, then let $\js'$ be the unique\footnote{Existence and uniqueness of this system are argued below the definition.} system equal to either $\js$ or one of its  descendant systems such that:

\begin{enumerate}
 \item $\branch$ has infinitely many occurrences of facts defined locally in $\js'$.\footnote{In more formal detail, there is some $a\in \Fd'$ s.t.\ $\branch$ contains infinitely many occurrences of $a$.}\label{propOne}
     \item No ancestor system of $\js'$ has property (a), i.e. for all ancestor systems of $\js$, $\branch$ has only finitely many occurrences of facts  defined locally in that system.
\end{enumerate}
Let $\branch'$ be the branch obtained from $\branch$ by removing all facts not defined locally in $\js'$.
Then $\be^*(\branch) = \be'(\branch')$, where $\be'$ is the branch evaluation of $\js'$.

  \end{enumerate}
\end{itemize}
\end{definition}

The conditions in Definition \ref{def:nested} guarantee that branches in
$\Merge(\js)$ are sequences of facts $x_0\rightarrow x_1\rightarrow \ldots$ s.t.\ $x_i$ is defined locally in a descendant or ancestor of 
$x_{i+1}$.
This guarantees that the system $\js'$ indeed exists and is unique: if a branch has infinitely many occurrences of facts defined (locally) in two systems, then it must also go infinitely often through a common ancestor.



\begin{example}[Example \ref{running:example:1} ctd.]
Consider again $\js$ as in Example \ref{running:example:1}.
We have, among others, the justification for $r$ in $\Merge(\js)$ on the left side of the following picture\footnote{The justification is visualized here as a graph; the actual justification is the tree-unravelling of this graph which will for instance have infinitely many nodes labeled $r$. For instance, the edge from $p$ to $r$ symbolises the fact that every node labelled $p$ has a child node labelled $r$.}; on the right side, the individual branches have been unravelled and the corresponding ``primed'' branches are drawn:
 
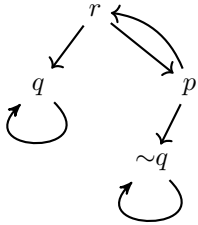
\begin{center}
    \begin{tikzpicture}[transform shape]

        \node (rt) at (-0.25,3) {$r$};
        \node (q) at (-1,2) {$q$};

        \node (p) at (1, 2) {$p$};
        \node (tq1) at (0.5, 1) {$\tild q$};
 \node (l1u) at (4,2) {$\branch_2$:};
 \node (r1u) at (5,2) {$r$};
 \node (q1u) at (6,2) {$q$};
 \node (q1u2) at (7,2) {$q$};
 \node (q1u3) at (8,2) {$\ldots$};

	 \node (l1ua) at (4,1) {$\branch_3$:};
 \node (r1ua) at (5,1) {$r$};
 \node (q1ua) at (6,1) {$p$};
 \node (q1u2a) at (7,1) {$\tild q$};
  \node (q1u3a) at (8,1) {$\tild q$};
 \node (q1u4a) at (9,1) {$\ldots$};

	 \node (l1ub) at (4,3) {$\branch_1$:};
 \node (r1ub) at (5,3) {$r$};
 \node (q1ub) at (6,3) {$p$};
 \node (q1u2b) at (7,3) {$r$};
  \node (q1u3b) at (8,3) {$p$};
 \node (q1u4b) at (9,3) {$\ldots$};

 \node (l1') at (4.2,-0.2) {$\branch'_1$:};
  \node (l2') at (6.2,-0.2) {$\branch'_2$:};
 \node (l3') at (8.2,-0.2) {$\branch'_3$:};

        \node (rtn) at (5,0) {$r$};
        \node (qn) at (7,0) {$q$};
        \node (tqn) at (9,0) {$\tild q$};


%
%
%

    \tikzstyle{EdgeStyle}=[style={->}]
	    \Edge(rt)(q)
	    \Edge(rt)(p)
	    \Edge(p)(tq1)
	    \Edge(r1u)(q1u)
	    	\Edge(q1u)(q1u2)
	    \Edge(q1u2)(q1u3)
	    \Edge(r1ua)(q1ua)
	    	\Edge(q1ua)(q1u2a)
	    \Edge(q1u2a)(q1u3a)
	    	    \Edge(q1u3a)(q1u4a)
	    	    	    \Edge(r1ub)(q1ub)
	    	\Edge(q1ub)(q1u2b)
	    \Edge(q1u2b)(q1u3b)
	    	    \Edge(q1u3b)(q1u4b)

	   \Loop[dist=1.0cm,dir=SO](rtn)(rtn)
	   	   \Loop[dist=1.0cm,dir=SO](q)(q)
	   	   \Loop[dist=1.0cm,dir=SO](tq1)(tq1)

 \Loop[dist=1.0cm,dir=SO](qn)(qn)
	   	   \Loop[dist=1.0cm,dir=SO](tqn)(tqn)
	       \tikzstyle{EdgeStyle}=[style={->, bend right}]
	    \Edge(p)(rt)

    \end{tikzpicture}

  \end{center}
 Notice that all literals in $\branch_1$
  occur infinitely many times in $\branch_1$. Since $r$ is defined locally in $\js$, i.e.\ it occurs in ${\cal F}_{dl}$,
 we have to look at $\js$ to evaluate $\branch_1$ with $\be^\star$. On the right side, the branch $\branch_1'=r\rightarrow r \rightarrow \ldots$ obtained from $\branch_1$ by removing all facts defined outside of $\js$ is given. We now obtain $\be^\star(\branch_1)$ by evaluating $\branch_1'$ according to $\branch_1$, i.e., $\be^\star(\branch_1)=\bekk(\branch_1')=\Un$.

 For  $\branch_2$
and $\branch_3$,
 we see that the highest justification system in which $q$ respectively $\tild q$ are defined is $\js^1$. We therefore obtain the branches $\branch_2'=q\rightarrow q\rightarrow \ldots$ and $\branch_3'=\tild q\rightarrow \tild q\rightarrow \ldots$. Since $\js^1$ uses the well-founded branch evaluation $\bewf$, we obtain $\be^\star(\branch_2)=\bewf(\branch_2')=\Fa$ and $\be^\star(\branch_3)=\bewf(\branch_3')=\Tr$. Here we can thus see the different branch evaluations of the component systems at work: for example, branches with infinite occurrences of positive literals are treated differently in view of the component justification system in which these literals are defined, and the branch evaluations used in these justification systems. In more detail, $\branch_1$ is evaluated to $\Un$ since it is
defined locally
 in $\js$, which uses the Kripke-Kleene evaluation, whereas $\branch_3$ is evaluated to $\Fa$ since it is defined in $\js^2$, which uses the well-founded evaluation.
 Altogether, we see that for $\interp(r)=\Tr$, $\interp(q)=\interp(p)=\Fa$,
 it holds that $\suppvalue(r, \interp)=\Tr$. In fact, it can be checked that $\interp$ is the unique $\Compress(\js)$-model, which is no coincidence as we will see next.
\end{example}

\section{Equivalence of Compression- and Merging-Based Characterisations}\label{sec:se}
Even though the compression-based and merging-based characterisation of semantics are based on quite different mechanisms, they give rise to the same evaluations for nested justification systems, i.e., they are equivalent. Showing this is rather involved, and due to spatial limitations, the details are given in appendix \ref{sec:appendix:th:treelikemergecompressstronglyequivalent}, and illustrated schematically in Figure \ref{fig:shrinkandexpand}. The crucial idea for the proof of equivalence is the definition of two operations, shrinking and expanding, that allow converting justifications in $\Compress(\js)$ to justifications in $\Merge(\js)$ and vice versa. On the basis of these operations, it is then shown that the supported value for defined facts are the same under $\Merge(\js)$ and $\Compress(\js)$ (for all defined facts $x$ and interpretations $\interp$):
  $$\suppvalue_{\!\!\Compress(\js)}^t(x,\interp) = \suppvalue_{\!\!\Merge(\js)}^t(x,\interp)$$
To guarantee that shrink and expand behave as expected, two minor assumptions are necessary, namely (1) that finite branches are mapped to their last element (\citet[Propositions 1 and 2]{ijcai/MarynissenBD21} have shown how to modify branch evaluations to satisfy this condition) and (2) that all parametric branch evaluations map infinite branches to logical facts (in fact, unless additional structure on the fact space is assumed, this will always be satisfied).
We obtain the following theorem.

\begin{theorem}\label{th:treelikemergecompressstronglyequivalent}
Let a compressible nested justification system $\js$ in which every branch evaluation maps finite branches to their last element, and infinite branches to logical facts. Then
$\Compress(\js)$ and $\Merge(\js)$ are equivalent.
\end{theorem}
It immediately follows that for a justification system that satisfies all of the preconditions of Theorem \ref{th:treelikemergecompressstronglyequivalent}, the $\Compress(\js)$-models coincide with the $\Merge(\js)$-models.

\begin{figure}
\begin{center}
  \begin{adjustbox}{max width=\textwidth}
    \begin{tikzpicture}[transform shape]

        \node (rt) at (-0.25,3) {$r$};
        \node (q) at (-1,2) {$q$};
        \node (p) at (1, 2) {$p$};
        \node (tq1) at (0.5, 1) {$\tild q$};
%


%
%
%

    \tikzstyle{EdgeStyle}=[style={->}]
	    \Edge(rt)(q)
	    \Edge(rt)(p)
	    \Edge(p)(tq1)

	   	   \Loop[dist=1.0cm,dir=SO](q)(q)
	   	   \Loop[dist=1.0cm,dir=SO](tq1)(tq1)


     \node (rtn) at (10,2) {$r$};
     \node (f) at (9.25,1) {$\Fa$};

        \node (t) at (10.75, 1) {$\Tr$};
	    \Edge(rtn)(f)
	    \Edge(rtn)(t)
	   \Loop[dist=1.0cm](rtn)(rtn)

  \node (pm) at (3.5,4) {$p$};
   \node (rm) at (4.25,3) {$r$};
   \node (tqm) at (2.75,3) {$\tild q$};
  \node (qm) at (5.5,3.5) {$q$};
  \node (under1) at (3.5,2.9) {};
  \node (under2) at (5.5,2.9) {};

\node (rule1) at (0,4) {$r\gets p,q$};
\node (rule2) at (10.2,4) {$r\gets r,\Tr,\Fa$};

%

        \node (ph) at (1.2,1.5) {};
        \node (phlower) at (1.25,1.4) {};
        \node (compress) at (5,4.7) {};
        \node (expand) at (4.8,1.7) {};
        \node (expand2) at (4.8,1.7) {};
        \node (expand1) at (4.3,1.7) {};
        \node (ph2) at (8.2,1.5) {};

        \node (ph3) at (8.2,1) {};
        \node (ph4) at (1.2,1) {};

		\node (ph3') at (4.41,1.5) {};
		\node (ph4') at (4.41,2.1) {};
		
		\node (ph5') at (4.41,3.7) {};
		\node (ph6') at (4.41,4.5) {};

		\node (ph5) at (9.5,4) {};
        \node (ph6) at (0.5,4) {};

    
%


\Edge(pm)(tqm)
	    \Edge(pm)(rm)
    	   	   \Loop[dist=1.0cm](tqm)(tqm)
	    	   	   \Loop[dist=1.0cm,dir=SO](qm)(qm)
	      \tikzstyle{EdgeStyle}=[style={->, bend right}]
	    \Edge(p)(rt)

	    	      \tikzstyle{EdgeStyle}=[style={->, double, thick, color=gray}]
	    \Edge[label=Expand](ph2)(ph)
   	    \Edge[label=Shrink](ph4)(ph3)

   	    \tikzstyle{EdgeStyle}=[style={,->, thick, color=gray, bend left=-30}]
   	    \Edge(ph6)(expand1)
   	    \tikzstyle{EdgeStyle}=[style={,-> , thick, color=gray, bend left=5}]
   	    \Edge(under1)(expand2)
   	    \tikzstyle{EdgeStyle}=[style={,-> , thick, color=gray, bend left=10}]
   	    \Edge(under2)(expand)

\tikzstyle{EdgeStyle}=[style={->, , thick, color=gray, bend left=20}]
   	    \Edge[label=Compress](ph6)(ph5)
   	    
\tikzstyle{EdgeStyle}=[style={,-> , thick, color=gray, bend left=30}]
   	    \Edge(pm)(compress)
   	    
   	    \tikzstyle{EdgeStyle}=[style={,-> , thick, color=gray, bend left=40}]
   	    \Edge(qm)(compress)

%

    \end{tikzpicture}
  \end{adjustbox}  \end{center}
\caption{
Schematic illustration (based on Example \ref{running:example:1}) of the \emph{shrinking}-operation, which converts a justification in $\Merge(\js)$ into a justification in  $\Compress(\js)$  (also illustrated schematically), and the \emph{expanding}-operation, which converts a justification in $\Compress(\js)$ (taking into account the relevant justifications in $\js$) into a justification in
$\Merge(\js)$.}
\label{fig:shrinkandexpand}
\end{figure}
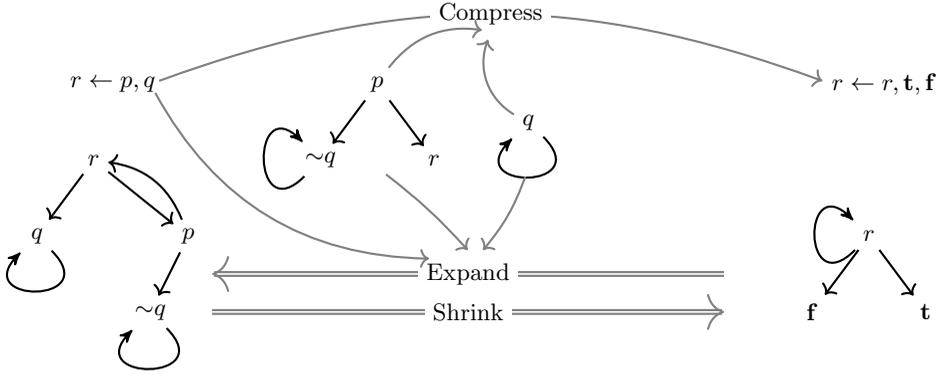

\begin{remark}\label{remark:compressible:compression}
As mentioned in Remark \ref{remark:compress:and:parametric}, compression was only defined for so-called compressible branch evaluations to allow evaluating the different involved systems separately. For merging, this restriction does not play a role: here all branch evaluations are taken into account simultaneously, in the definition of $\be^*$. As such, we have not just given a novel characterization of the semantics of nested systems, but also significantly expanded its scope, by allowing non-parametric systems, such as stable or supported, to be nested.
\end{remark}

\section{\FOFD: Application to Fixpoint Definitions}\label{sec:fo}

$\FOFD$ \cite{tplp/HouDD10}  is a logic that integrates fixpoint definitions based on rules with first-order logic. Essentially, a \emph{least fixpoint definition} (respectively \emph{greatest fixpoint definition}) is defined inductively as an expression of the form  $\mathcal{D}=
\left\lfloor \mathcal{R}, \Delta_1, \ldots, \Delta_m, \nabla_1, \ldots, \nabla_n \right\rfloor$ (respectively $\left\lceil \mathcal{R}, \Delta_1, \ldots, \Delta_m, \nabla_1, \ldots, \nabla_n \right\rceil$), where $\mathcal{R}$ is a set of rules
and each $\Delta_i$ is a least fixpoint definition
and each $\nabla_j$ is a greatest fixpoint definition
, s.t.\ every symbol 
is defined in at most one of ${\cal R}$, $\Delta_1$, $\ldots$, $\Delta_m$, $\nabla_1$, $\ldots$, $\nabla_n$.
Additionally, defined symbols are only allowed to occur positively in rule bodies. The semantics of $\FOFD$, explained in full detail in Appendix \ref{appendix:fo(fd)}, is given in terms of (two-valued) interpretations, and is defined iteratively, by means of an immediate consequence operator $\Gamma^{\cal D}(I)$. Consistent with the idea of capturing a definition, a unique model according to these semantics is guaranteed to exist. The nested nature of fixpoint definitions comes into play when requirements of different levels conflict with each-other. In that case, intuitively, the highest level (the outermost definition) will be given priority.


As an example, restricted to the propositional case, consider:
  \begin{equation*}
 {\cal D}= \leastframe{
p\gets q\lor r; \quad q\gets p; \quad u\gets s\\
    \greatestframe{
     r\gets p;\quad s\gets t\lor q;\quad t\gets s
    }
  }
  \end{equation*}
  This fixpoint definition consists of two levels. The uppermost or global level is a least fixpoint definition of $p$, $q$ and $u$, whereas the lower or inner level contains a greatest fixpoint definition of the atoms $r$, $s$ and $t$. Intuitively, the semantics will ensure that the model of this definition corresponds to a greatest fixpoint of the immediate consequence operator based on the inner system, and a least fixpoint of the immediate consequence operator based on the outer system, together with the respective interpretation for the inner system. For example, the interpretation that assigns $\Tr$ to $s$, $t$ and $u$, and $\Fa$ to $p$, $q$ and $r$ to $\Fa$ is a model of this fixpoint definition.
  We see that the cycle between $p\gets q\lor r$ and $r\gets p$ gives rise to a
  conflict between the two levels: the greatest fixpoint nature of the lowest level would require us to make $p$ and $r$ true, whereas the least fixpoint nature of the highest level would require us to make $p$ and $r$ false. Since priority is given to the highest level, $p$ and $r$ are false in the defined interpretation.

 Nested justifications were partially inspired by $\FOFD$, and is therefore not surprising that \FOFD can be captured using nested justifications.
  This way of assigning priority to higher levels in case of conflicting demands also shows up in the
merging-based characterisation, where the highest component system is used to evaluate infinite branches.

%
For the propositional case, the translation is rather straightforward. Given a fixpoint definition ${\cal D}$, $\js_{\cal D}$ is defined as $\langle {\cal F},{\cal F}_d, \Fdl, \rules, \be, \{\js_{\Delta_1},\ldots,\js_{\Delta_m},\js_{\nabla_1},\ldots,\js_{\nabla_n}\}\rangle$, where $\be=\bewf$ if ${\cal D}$ is a least fixpoint definition and $\be=\becwf$ if ${\cal D}$ is a greatest fixpoint definition. We can thus represent the nested justification system $\js_{{\cal D}}$ as:
$$
\bewf:
\jframe{
p\gets q\lor r; \quad q\gets p; \quad u\gets s\\
\becwf:
\jframe{
\begin{matrix*}[l]
     r\gets p;\quad s\gets t; \quad s\gets q;\quad t\gets s
\end{matrix*}
}
}
$$
This translation is  adequate, in the sense that a unique $\bewf$-model (respectively $\becwf$-model) that leaves no atoms $\Un$ is guaranteed to exist and corresponds to the least (respectively greatest) fixpoint of the fixpoint definition.
Notice that we replaced $s\gets t\lor q$ by $s\gets t$ and $s\gets q$. For more complex formulae, more work is required (see Appendix \ref{appendix:fo(fd)}).
For the full first-order case, the idea is essentially the same, but an intermediary system is included between $\js_{\cal D}$ and $\js_{\Delta_1},\ldots,\js_{\Delta_m},\js_{\nabla_1},\ldots,\js_{\nabla_n}$ to ensure that first-order formulae are treated adequately. This translation, as well as the correspondence results and the technical details on fixpoint definitions, are detailed in the appendix \ref{appendix:fo(fd)}.


\section{Conclusion, in View of Related Work}\label{sec:rel:work}
In this paper, we gave two characterisations of the semantics of nested justification systems, and we have shown that these characterisations are equivalent. Furthermore, we have shown how nested justifications can capture nested fixpoint definitions.

The potential applications of nested justification systems are extensive, as they allow for the modular, and therefore straightforward, design of rule-based languages. A prime example is the definition of semantics for aggregates. Nested justification systems allow to separate the definition of aggregates from that of a logic program, by adding
 rules defining the aggregates at the lowest level of a justification system (cf.\  Example \ref{ex:aggregates:1}).
In future work, we plan to investigate exact translations from known semantics for logic programs with aggregate expressions (e.g., \cite{tplp/PelovDB07,faber2011semantics,gelfond2019vicious,corr/AlvianoFG21,VanbesienBD21} in nested justification systems, bringing justification theory's explanatory potential to such semantics.

Some aspects of nested justification, in particular the fact that different modules using different semantics can be combined, are reminiscent of \emph{multi-context systems} \cite{brewka2018reactive}. There are, however, significant differences between nested justification systems and multi-context systems.
On the one hand, multi-context systems allow for a broader range of knowledge bases to be used as modules. For example, default rules or autoepistemic belief bases can be used in multi-context systems, but these cannot be modelled straightforwardly in justification theory. On the other hand, nested justification systems allow for more sophisticated interactions and structures. 
 It remains an open question whether, for a sub-class of multi-context systems based on rule-based contexts, a translation of multi-context systems into nested justification systems exists.
In this paper, we gave a new view on nested justification systems that retains the justification quality and showed that under very mild restrictions, the two views are equivalent for tree-like justifications.
Whether this holds for graph-like justifications is an open question.

Another question that shows up in several papers on justification theory \cite{DeneckerPhD93,nmr/MarynissenPBD18,tplp/MarynissenBD20} is the \emph{consistency} question: is it always so that $\suppvalue(\tild x,\interp) = \tild \suppvalue(x,\interp)$?
The same question is relevant for nested systems. However, in a companion paper \cite{treeLikeConsistent}, we show that in the tree-like system, all branch evaluations (and hence also our newly defined \emph{merge} evaluation) are consistent.
And because of our equivalence, we have that compression is also consistent.


\bibliographystyle{tlplike}

\newpage

\begin{appendix}
\section{Additional Preliminaries}\label{appendix:additional:preli}
A notion that will be useful is complementation and the derived concept of complementarity \cite{nmr/MarynissenPBD18} . 
It is a generic mechanism that allows turning a set of rules for $x$ into a set of rules for $\tild x$.


Before defining complementation, we first define \emph{selection functions} for $x$.
A selection function for $x$ is a mapping $s\colon \jf(x) \rightarrow \F$ such that $s(A) \in A$ for all rules of the form $x \gets A$.
Intuitively, a selection function chooses an element from the body of each rule of $x$.
For a selection function $s$, the set $\{s(A) \mid A \in \jf(x)\}$ is denoted by $\im(s)$.
Given a set of rules $\rules$, we define $\rules^*$ to be the set of rules of the form $\tild x\gets \tild \im(\sel)$ for $x\in \Fd$ that has rules in $\rules$ and $s$ a selection function for $x$. The complementation of $\jfcomplete$ is 
A justification frame $\jfcomplete$ is  \emph{complementary} if it is fixed under comlementation (i.e.\ $\rules\cup \rules^*=\rules$).

%

\section{Proofs for Section \ref{sec:just:systems}}
\noindent
{\bf Lemma \ref{lemma:stratified:facts}.}
  Let $\js=\nestedjs$ be a nested justification system.
  If $i \neq j$, then $\Fd^i \cap \F^j = \emptyset$.
\begin{proof}
 Take $i \neq j$.
 Because $\Fd$ is partitioned into $\setl{\Fdl, \Fd^1, \ldots, \Fd^k}$, it suffices to prove that $\Fd^i \cap \Fo^j = \emptyset$.
 This follows from \eqref{def:openfactsofnestedsystem} in \cref{def:nested}.
\end{proof}

To make the nesting structure more apparent, nested systems can also be viewed as a tree of unnested justification systems.
\begin{proposition}\label{prop:tree}
  A nested justification system $\js=\nestedjs$ corresponds to a tree of justification systems, called the \emph{nesting tree} of $\js$, such that
  \begin{enumerate}
    \item The root of the tree is $\langle \F, \Fdl, \rules, \be \rangle$;
    \item Each element in $\Fd$ is a defined fact in exactly one system;
    \item Each element in $\Fo$ is an open fact in every system.
    \item A fact defined in a system $\js_1$ occurs as open fact in another system $\js_2$ if and only if $\js_2$ is an ancestor or descendant of $\js_1$;
    \item The root has $k$ children and for each $1 \leq i \leq k$ the subtree rooted in the $i$th child of the root is the nesting tree of $\js^i$.
  \end{enumerate}
\end{proposition}
\begin{proof}
By combining the first and last point, the tree is uniquely determined. The rest of the conditions follow from the definition of nested justification systems.
\end{proof}

\section{Proofs for Section \ref{sec:compress}}\label{appendix:section:compress}

\noindent
{\bf Proposition \ref{prop:supportinvariantunderflat}.}
   Let $\js_1$ and $\js_2$ be two parametric justification systems that map $x \rightarrow y$ to $y$. Then:
\begin{enumerate}
\item $\js_1$ is equivalent to $\Flat(\js_1)$.
\item  Then $\js_1$ and $\js_2$ are  equivalent if and only if $\Flat(\js_1)$ and $\Flat(\js_2)$ are equivalent.
\end{enumerate}   
\begin{proof}
  Take an $\F$-interpretation $\interp$ and $x \in \Fd$.
  Let $J$ be a $\js$-justification such that $\suppvalue_{\js}(x,\interp) = \jval(J,x,\interp)$.
  Then by construction of $\Flat(\js)$, this $J$ corresponds to a rule $x \gets A$ in $\Flat(\js)$ such that $A \subseteq \Fo$ and for all $y \in A$ it holds that $\branch(x \rightarrow y) \geqt \jval(J,x,\interp)$.
  The rule $x \gets A$ is a locally complete justification in $\Flat(\js)$.
  Therefore, $\suppvalue_{\Flat(\js)}(x,\interp) \geqt \jval(J,x,\interp) = \suppvalue_{\js}(x,\interp)$.
  Similarly, every locally complete justification in $\Flat(\js)$ with $x$ as root is a single rule $x \gets A$ such that $A = \setprop{\be(\branch)}{\branch \in \branches_J(x)}$ for some locally complete justification $J$ in $\js$.
  Therefore, $\suppvalue_{\js}(x,\interp) \geqt \suppvalue_{\Flat(\js)}(x,\interp)$.
\end{proof}

\section{Proof of Theorem \ref{th:treelikemergecompressstronglyequivalent}}\label{sec:appendix:th:treelikemergecompressstronglyequivalent}
In this section we show Theorem \ref{th:treelikemergecompressstronglyequivalent}. The theorem follows immediately from two results shown in this appendix, 
Corollary \ref{cor:mergelessthancompress} and Corollary \ref{cor:compresslessthanmerge}. Corollary \ref{cor:mergelessthancompress} is shown in Section \ref{sec:shrinking:justs}, and Corollary \ref{cor:compresslessthanmerge} is shown in Section \ref{sec:expand}. We also provide additional examples and explanation for the two operations of shrinking and expansion in the respective subsections.


\subsection{Shrinking Justifications}\label{sec:shrinking:justs}
To prove the equivalence between $\Compress(\js)$ and $\Merge(\js)$ we need a way to convert justifications from one system to the other.
Going from $\Merge(\js)$ to $\Compress(\js)$ should be relatively easy since justifications in $\Merge(\js)$ contain more information.
By removing this extra information, we can reduce the justification.
This is illustrated in the next example.
\begin{example}
  Take the nested system
  \begin{equation*}
  \jframe{
    z \gets x, y \\
    y \gets w \\
    \jframe{
      x \gets x \\
      w \gets a, b
    }
  },
  \end{equation*}
  where both levels are evaluated with $\bewf$.
  Let us look at the following justification for $z$ in the merge system.
  
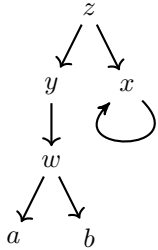
\begin{center}
    \begin{tikzpicture}[transform shape]
    \node (z) at (1.5, 3) {$z$};
    \node (x) at (2,2) {$x$};
    \node (y) at (1,2) {$y$};
    \node (w) at (1,1) {$w$};
    \node (a) at (0.5,0) {$a$};
    \node (b) at (1.5,0) {$b$};
    \tikzstyle{EdgeStyle}=[style={->}]
    \Edge(z)(x)
    \Edge(z)(y)
    \Edge(y)(w)
    \Edge(w)(a)
    \Edge(w)(b)
    \Loop[dist=1.0cm,dir=SO](x)
    \end{tikzpicture}
  \end{center}
  The loop on the right side is evaluated to $\Fa$.
  So to eliminate $x$ from the justification, we can replace this loop with the fact $\Fa$.
  On the left side, $w$ can be replaced by $a$ and $b$.
  This will produce the justification
  
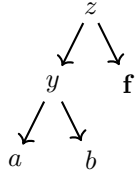
\begin{center}
    \begin{tikzpicture}[transform shape]
    \node (z) at (1.5, 2) {$z$};
    \node (f) at (2,1) {$\Fa$};
    \node (y) at (1,1) {$y$};
    \node (a) at (0.5,0) {$a$};
    \node (b) at (1.5,0) {$b$};
    \tikzstyle{EdgeStyle}=[style={->}]
    \Edge(z)(f)
    \Edge(z)(y)
    \Edge(y)(a)
    \Edge(y)(b)
    \end{tikzpicture}
  \end{center}
  This is a justification of the compression
  \begin{equation*}
  \jframe{
    z \gets \Fa, y \\
    y \gets a,b
  }.
  \end{equation*}
  A justification in the merge also contains justifications in the unnested systems of the nesting tree.
  In the example above, these are
  
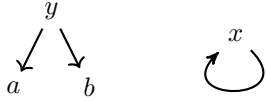
\begin{center}
    \begin{tikzpicture}[transform shape]
    \node (w) at (1,1) {$y$};
    \node (a) at (0.5,0) {$a$};
    \node (b) at (1.5,0) {$b$};
    \tikzstyle{EdgeStyle}=[style={->}]
    \Edge(w)(a)
    \Edge(w)(b)
    \end{tikzpicture}
    \qquad
    \begin{tikzpicture}[transform shape]
    \node (x) at (0,0) {$x$};
    \tikzstyle{EdgeStyle}=[style={->}]
    \Loop[dist=1.0cm,dir=SO](x)
    \end{tikzpicture}
  \end{center}
  These justifications can be reduced to a single set, by evaluating their branches, in the same manner as in the flattening.
  Then replacing facts of lower systems by such a set, we get a justification in compression as above.
\end{example}

The idea in this example is formalised in the \emph{shrinking} of a justification.
In the example above, we evaluated certain subjustifications.
\begin{definition}
  Let $\js=\nestedjs$ and $J$ a locally complete justification in $\Merge(\js)$.
  Let $\js^*$ be a node of the nesting tree of $\js$ that has children, but no grandchildren.
  Let $\js^1, \ldots, \js^m$ be the children of $\js^*$ and let $X$ be the set of internal nodes of $J$ that are defined in some $\js^i$.
  For each $x \in X$ defined in $\js^i$, let $P(x)$ be the set of maximal $J$-paths starting with $x$ that consist of
  \begin{itemize}
    \item an infinite number of defined facts of $\js^i$, or
    \item a finite number of defined facts of $\js^i$ and an open fact in $\js^i$.
  \end{itemize}
  This means that $P(x)$ forms a justification $J_x$ inside $\js^i$.
  The justification $J_x$ is called the \emph{subjustification} for $x$ in $J$.
\end{definition}

In correspondence with the flattening in the compression, each of the subjustifications can be reduced to a single rule.

\begin{definition}\label{def:shrinklocal}
  Let $\js=\nestedjs$ and $J$ a locally complete justification in $\Merge(\js)$.
  Let $\js^*$ be a node of the nesting tree of $\js$ that has children, but no grandchildren.
  Let $\js^1, \ldots, \js^m$ be the children of $\js^*$ and let $X$ be the set of internal nodes of $J$ that are defined in some $\js^i$.
  Each subjustification $J_x$ for $x$ in $J$ corresponds to a rule $x \gets B_x$ in $\Flat(\js^i)$.
  Define $\shrink_{\js^*}(J)$ to be the justification obtained from $J$ where
  \begin{itemize}
    \item every rule $y \gets A$ in $J$ is retained if $y$ is defined outside $\js^*$ or its children.
    \item every rule $y \gets A$ in $J$ with $y$ defined in $\js^*$ is replaced by
    $y \gets A \setminus X \cup \cup_{x \in A \cap X} B_x$;
    \item together with the rules $x \gets B_x$.
  \end{itemize}
\end{definition}

The construction of $\shrink$ follows the construction of compression: the evaluation of the subjustifications is in line with the flattening, while the shrinking itself is like the unfolding.
The addition of the rules $x \gets B_x$ corresponds to adding the rules of the flattening in the compression.

When shrinking we start with a locally complete justification, and so we would like to end up with a locally complete justification.
If the shrinking is not locally complete, then there is a subjustification $J_x$ that has a branch $\branch$ starting with $x$ that is mapped to a defined fact that is not internal in $\shrink_{\js^*}(J)$.
If $\branch$ is finite, then by locally completeness of $J$, the last element of $\branch$ is also internal in $J$.
Therefore, it is only possible that $\branch$ is infinite.
All branch evaluations considered are parametric, and thus $\branch$ is mapped to an open fact.
Branch evaluations tend to map branches to some element in the branch (or the negation) or some logical fact.
For a branch evaluation to map to another fact requires for that fact to have a special status in the given justification frame, by adding extra structure on the underlying facts space.
This is because branch evaluations are defined independent of justification frames.
Since we are dealing with parametric branch evaluations, we cannot map to elements of an infinite branch since they are all defined.
Therefore, the only feasible image of an infinite branch is a logical fact.

\begin{lemma}
  If the branch evaluation of the children of $\js^*$ map infinite branches into $\lf$, then $\shrink_{\js^*}(J)$ is locally complete.
\end{lemma}
\begin{proof}
  If $\shrink_{\js^*}(J)$ has a defined leaf, then there is a finite branch $\branch^*$ in $\shrink_{\js^*}(J)$ ending in a defined leaf.
  This means that there is either a finite $J$-branch ending in the same leaf, or an infinite $J$-branch for which the highest level with an infinite number of elements is a child of $\js^*$ and evaluating the restriction of this branch to these elements will produce this defined leaf.
  The former is not possible since $J$ is locally complete and the latter is not possible by our assumptions.
\end{proof}

This property will also be needed when we do the reverse construction of shrinking.
Let us provide another example.
\begin{example}
  Take the following nested system
  \begin{equation*}
  \jframe{
    z \gets y \\
    w \gets x \\
    \jframe{
      y \gets x \\
      x \gets w, y
    }
  },
  \end{equation*}
  where every level is evaluated under $\bewf$.
  Take the following justification for $z$ in the $\Merge$:
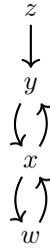
  \begin{center}
    \begin{tikzpicture}[transform shape]
    \node (z) at (0,3) {$z$};
    \node (y) at (0,2) {$y$};
    \node (x) at (0,1) {$x$};
    \node (w) at (0,0) {$w$};
    \tikzstyle{EdgeStyle}=[style={->}]
    \Edge(z)(y)
    \tikzstyle{EdgeStyle}=[style={->}, bend right]
    \Edge(x)(w)
    \Edge(w)(x)
    \Edge(y)(x)
    \Edge(x)(y)
    \end{tikzpicture}
  \end{center}
  The subjustifications for $x$ and $y$ are the same and equal to
  
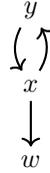
\begin{center}
    \begin{tikzpicture}[transform shape]
    \node (y) at (0,2) {$y$};
    \node (x) at (0,1) {$x$};
    \node (w) at (0,0) {$w$};
    \tikzstyle{EdgeStyle}=[style={->}]
    \Edge(x)(w)
    \tikzstyle{EdgeStyle}=[style={->}, bend right]
    \Edge(y)(x)
    \Edge(x)(y)
    \end{tikzpicture}
  \end{center}
  By evaluating the branches starting with $y$ (respectively $x$), we get rules $y \gets \Fa, w$ and $x \gets \Fa, w$ in the flattening of the inner system.
  Shrinking the justification for $z$ will produce (the rules for $x$ and $y$ are left out for clarity)
  
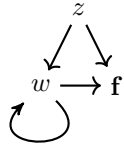
\begin{center}
    \begin{tikzpicture}[transform shape]
    \node (z) at (0.5, 1) {$z$};
    \node (w) at (0,0) {$w$};
    \node (f) at (1,0) {$\Fa$};
    \tikzstyle{EdgeStyle}=[style={->}]
    \Edge(z)(w)
    \Edge(z)(f)
    \Edge(w)(f)
    \Loop[dist=1.0cm,dir=SO](w)
    \end{tikzpicture}
  \end{center}
\end{example}

In both examples above, we had a single nesting.
If multiple nesting occur, it is not directly clear in which system the shrunken justification is in.
For example, if our nesting tree looks like
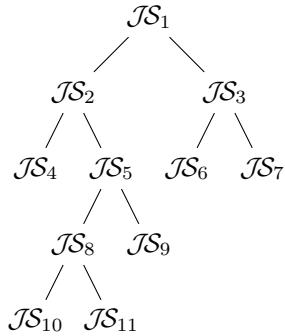
\begin{center}
  \begin{tikzpicture}[level distance=1cm,
  level 1/.style={sibling distance=2cm},
  level 2/.style={sibling distance=1cm}]
  \node {$\js_1$}
  child {node {$\js_2$}
    child {node {$\js_4$}}
    child {node {$\js_5$}
      child { node {$\js_8$}
        child {node {$\js_{10}$}}
        child {node {$\js_{11}$}}
      }
      child {node {$\js_9$}}
    }
  }
  child {node {$\js_3$}
    child {node {$\js_6$}}
    child {node {$\js_7$}}
  };
  \end{tikzpicture}
\end{center}
and we want to shrink with respect to $\js_8$, then the shrinking brings us in the merge of the nested system with the nesting tree
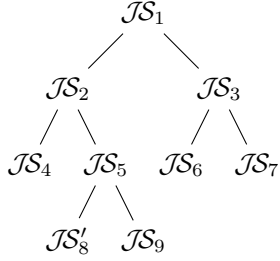
\begin{center}
  \begin{tikzpicture}[level distance=1cm,
  level 1/.style={sibling distance=2cm},
  level 2/.style={sibling distance=1cm}]
  \node {$\js_1$}
  child {node {$\js_2$}
    child {node {$\js_4$}}
    child {node {$\js_5$}
      child { node {$\js_8'$} }
      child {node {$\js_9$}}
    }
  }
  child {node {$\js_3$}
    child {node {$\js_6$}}
    child {node {$\js_7$}}
  };
  \end{tikzpicture}
\end{center}
where $\js_8'$ is the compression of the system with the nesting tree
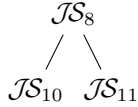
\begin{center}
  \begin{tikzpicture}[level distance=1cm,
  level 1/.style={sibling distance=1cm},
  level 2/.style={sibling distance=1cm}]
  \node {$\js_8$}
  child {node {$\js_{10}$}}
  child {node {$\js_{11}$}};
  \end{tikzpicture}
\end{center}

Shrinking a justification will not produce a justification with a worse value if we demand that every finite branch should be mapped to its last element by all branch evaluations involved.
\begin{proposition}\label{prop:shrinkingretainsvalue}
  Let $\js$ be a nested justification system with all branch evaluations involved parametric, mapping infinite branches to logical facts, and finite branches to their last element.
  Let $\js^*$ be a justification system in the nesting tree of $\js$ for which we can shrink.
  Let $\shrink_{\js^*}(\js)$ be the nested justification system after shrinking.
  Take a justification $J$ for $x$ in $\Merge(\js)$.
  Then
  $\jval_{\Merge(\js)}(J,x,\interp) = \jval_{\Merge(\shrink_{\js^*}(\js))}(\shrink_{\js^*}(J),x,\interp)$
  for all $\F$-interpretations $\interp$.
\end{proposition}
\begin{proof}
  There is a correspondence between branches of $J$ and $\shrink_{\js^*}(J)$.
  Every $J$-branch $\branch$ can be shrunk to a $\shrink_{\js^*}(J)$-branch $\branch^*$.
  Likewise, every $\shrink_{\js^*}(J)$-branch $\branch^*$ comes from a $J$-branch $\branch$.
  Moreover, $\be(\branch) = \be^*(\branch^*)$.
  Indeed, take an arbitrary $J$-branch $\branch$.
  For finite branches, this is obvious; so assume it is infinite.
  It it contains infinite facts from a system higher than the children of $\js^*$, then equality is also obvious.
  Therefore, we can assume that $\branch$ contains infinite facts from some child $\js^i$ of $\js^*$.
  This means $\branch^*$ is finite and its evaluation under $\be^*$ is equal to $\be(\branch)$.
\end{proof}

So far, we only shrink a single level, but by applying the shrink operation iteratively, a justification in $\Merge(\js)$ is transformed to a justification in $\Compress(\js)$.
\begin{definition}
  Let $J$ be a justification in $\Merge(\js)$.
  Choose a justification system $\js^*$ with children, but not grandchildren in the nesting tree of $\js$.
  Then $\shrink_{\js^*}(\js)$ will have a smaller nesting tree than $\js$.
  By iterating this procedure until the nesting tree is a single node, we get a justification in $\Compress(\js)$, which we will denote by $\shrink(J)$.
\end{definition}

The above definition is well-defined since the order of the local shrinking operations does not matter.
This is because the local shrinking operation only affects the rules for facts defined in $\js^*$ or its children.
By \cref{prop:shrinkingretainsvalue}, the value of $J$ and $\shrink(J)$ are equal.
This solves one side of the equivalence.
\begin{corollary}\label{cor:mergelessthancompress}
  $\suppvalue_{\Merge(\js)}(x,\interp) \leqt \suppvalue_{\Compress(\js)}(x,\interp)$ for all defined facts $x$ and interpretations $\interp$.
\end{corollary}
\begin{proof}
  For every justification $J$ in $\Merge(\js)$, we can construct an equally good justification in $\Compress(\js)$: $\shrink(J)$.
\end{proof}

\subsection{Expanding Justifications}\label{sec:expand}

The other direction is not as straightforward since we have to start from a justification in $\Compress$ and we have to `inflate' it to a justification of $\Merge$.
One way to achieve this, is by `pasting' in subjustifications.


\begin{example}\label{ex:pastingordermattersforexpand}
  Take the nested system
  \begin{equation*}
  \jframe{
    x \gets a \\
    y \gets b \\
    \jframe{
      a \gets x \\
      a \gets \Tr \\
      b \gets a
    }
  }.
  \end{equation*}
  The compression is equal to
  \begin{equation*}
  \jframe{
    x \gets x \\
    x \gets \Tr \\
    y \gets x \\
    y \gets \Tr
  }.
  \end{equation*}
  Take the following justification for $y$ in the compression:
  \begin{center}
    \begin{tikzpicture}[transform shape]
    \node (y) at (0,2) {$y$};
    \node (x) at (0,1) {$x$};
    \node (t) at (0,0) {$\Tr$};
    \tikzstyle{EdgeStyle}=[style={->}]
    \Edge(y)(x)
    \Edge(x)(t)
    \end{tikzpicture}
  \end{center}
  The rule for $y$ in $J$ comes from the rule $z \gets b$ in the top level and the subjustification
  \begin{center}
    \begin{tikzpicture}[transform shape]
    \node (b) at (0,2) {$b$};
    \node (a) at (0,1) {$a$};
    \node (x) at (0,0) {$x$};
    \tikzstyle{EdgeStyle}=[style={->}]
    \Edge(b)(a)
    \Edge(a)(x)
    \end{tikzpicture}
  \end{center}
  On the other hand, the rule for $x$ in $J$ comes from the rule $x \gets a$ in the top level and the subjustification
  \begin{center}
    \begin{tikzpicture}[transform shape]
    \node (a) at (0,1) {$a$};
    \node (t) at (0,0) {$\Tr$};
    \tikzstyle{EdgeStyle}=[style={->}]
    \Edge(a)(t)
    \end{tikzpicture}
  \end{center}
  These subjustifications are not compatible: the rules for $a$ are different.
  So putting the subjustifications into the justification for $y$ in the compression produces a justification in the merge.
  
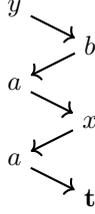
\begin{center}
    \begin{tikzpicture}[transform shape]
    \node (y) at (0,2.5) {$y$};
    \node (b) at (1,2) {$b$};
    \node (a) at (0,1.5) {$a$};
    \node (x) at (1,1) {$x$};
    \node (a2) at (0,0.5) {$a$};
    \node (t) at (1,0) {$\Tr$};
    \tikzstyle{EdgeStyle}=[style={->}]
    \Edge(y)(b)
    \Edge(b)(a)
    \Edge(a)(x)
    \Edge(x)(a2)
    \Edge(a2)(t)
    \end{tikzpicture}
  \end{center}
  In this example, we can still produce an equally good justification by pasting the subjustifications together in the correct order to get the subjustification
  \begin{center}
    \begin{tikzpicture}[transform shape]
    \node (b) at (0,2) {$b$};
    \node (a) at (0,1) {$a$};
    \node (t) at (0,0) {$\Tr$};
    \tikzstyle{EdgeStyle}=[style={->}]
    \Edge(b)(a)
    \Edge(a)(t)
    \end{tikzpicture}
  \end{center}
  Pasting this subjustification for the rule $y \gets x$ we get the justification in the merge
  
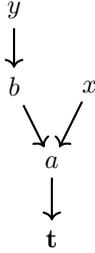
\begin{center}
    \begin{tikzpicture}[transform shape]
    \node (y) at (0,3) {$y$};
    \node (b) at (0,2) {$b$};
    \node (x) at (1,2) {$x$};
    \node (a) at (0.5,1) {$a$};
    \node (t) at (0.5,0) {$\Tr$};
    \tikzstyle{EdgeStyle}=[style={->}]
    \Edge(y)(b)
    \Edge(b)(a)
    \Edge(x)(a)
    \Edge(a)(t)
    \end{tikzpicture}
  \end{center}
  Shrinking this justification produces an equally good justification as $J$
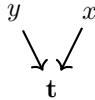
  \begin{center}
    \begin{tikzpicture}[transform shape]
    \node (y) at (0,1) {$y$};
    \node (x) at (1,1) {$x$};
    \node (t) at (0.5,0) {$\Tr$};
    \tikzstyle{EdgeStyle}=[style={->}]
    \Edge(y)(t)
    \Edge(x)(t)
    \end{tikzpicture}
  \end{center}
\end{example}

For a justification $J$ in $\Merge(\js)$ it suffices to construct a justification $J^*$ in $\Merge(\js)$ such that $\shrink(J^*) = J$.
This would imply that $\jval_{\Compress(\js)}(J,x,\interp) = \jval_{\Merge(\js)}(J^*,x,\interp)$ and solve the missing part of the equivalence between $\Compress(\js)$ and $\Merge(\js)$.

As with $\shrink$, we will first define the operation locally.
So let $\js$ be a nested justification system so that the unnested leaves of its nesting tree are compressions of nested justification systems.\footnote{These are exactly the systems one can get from shrinking a nested justification system.}
Let $\js^*$ be such a leaf which is a compression of a nested system with depth $> 1$.
Take a justification $J$ in $\Merge(\js)$.
Suppose $\js^*$ is the compression of $\nestedjs$.
We define $\expand(J)$ to be a justification in $\Merge$ of the system obtained from $\js$ with $\js^*$ replaced by the nested system with nesting tree
\begin{center}
  \begin{tikzpicture}[level distance=1cm,
  level 1/.style={sibling distance=1cm},
  level 2/.style={sibling distance=1cm}]
  \node {$\js_t$}
  child {node {${\js^1}'$}}
  child {node {$\ldots$}}
  child {node {${\js^k}'$}};
  \end{tikzpicture}
\end{center}
where $\js_t$ is the top level of $\js^*$ and ${\js^i}'$ is the compression of $\js^i$.

Any $z \gets A$ in $J$ with $z$ defined in $\js^*$ corresponds to a rule $z \gets B$ in $\js_t$ such that for a subset $X$ of $B$ and all $x \in X$, there is a justification $J_x$ in some $\js^i$ for $x$ such that $A = B \setminus X \cup \setprop{\setprop{\be(\branch)}{\branch \in \branches_{J_x}(x)}}{x \in X}$.
By replacing $z \gets A$ with $z \gets B$ and putting $J_x$ under $x \in X$ we get a justification $\expand_{\js^*}(J)$.

Remark that the property that infinite branches are mapped to logical facts is necessary here.
If not, then the construction above can put an infinite branch in between an infinite branch.
This is will not create a branch, but a sequence with a higher ordinal number than $\omega$.

Shrinking the obtained justification returns the original justification.

\begin{proposition}\label{prop:shrinkofexpandisidentity}
  Let $J$ be a justification.
  Then $\shrink_{\js_t}(\expand_{\js^*}(J)) = J$.
\end{proposition}
\begin{proof}
  Shrinking $\expand_{\js^*}(J)$ corresponds to removing the justifications $J_x$ and replacing $z \gets B$ with $z \gets A$ as in the construction above.
\end{proof}
\begin{corollary}\label{cor:expandretainsvalue}
  Let $J$ be a justification.
  Then $\jval(J,x,\interp) = \jval(\expand_{\js^*}(J),x,\interp)$.
\end{corollary}
\begin{proof}
  This follows directly from \cref{prop:shrinkofexpandisidentity,prop:shrinkingretainsvalue}.
\end{proof}

If multiple nestings are present, it can become confusing in which system $\expand(J)$ is a justification.
The following example provides some clarification.
\begin{example}
  Let $\js$ be the nested justification system with nesting tree
  
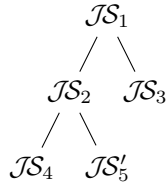
\begin{center}
    \begin{tikzpicture}[level distance=1cm,
    level 1/.style={sibling distance=1cm},
    level 2/.style={sibling distance=1cm}]
    \node {$\js_1$}
    child {node {$\js_2$}
      child {node {$\js_4$}}
      child {node {$\js_5'$}}
    }
    child {node {$\js_3$}};
    \end{tikzpicture}
  \end{center}
  where $\js_5'$ is the compression of the system with nesting tree
  
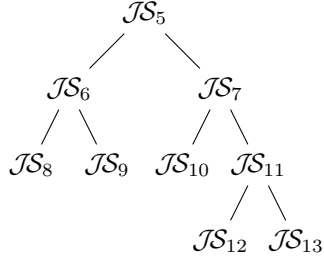
\begin{center}
    \begin{tikzpicture}[level distance=1cm,
    level 1/.style={sibling distance=2cm},
    level 2/.style={sibling distance=1cm}]
    \node {$\js_5$}
    child {node {$\js_6$}
      child {node {$\js_8$}}
      child {node {$\js_9$}}
    }
    child {node {$\js_7$}
      child {node {$\js_{10}$}}
      child {node {$\js_{11}$}
        child {node {$\js_{12}$}}
        child {node {$\js_{13}$}}
      }
    };
    \end{tikzpicture}
  \end{center}
  Expanding a justification in $\Merge(\js)$ gets us in the $\Merge$ of the system with nesting tree
  
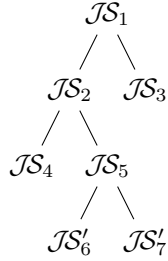
\begin{center}
    \begin{tikzpicture}[level distance=1cm,
    level 1/.style={sibling distance=1cm},
    level 2/.style={sibling distance=1cm}]
    \node {$\js_1$}
    child {node {$\js_2$}
      child {node {$\js_4$}}
      child {node {$\js_5$}
        child {node {$\js_6'$}}
        child {node {$\js_7'$}}
      }
    }
    child {node {$\js_3$}};
    \end{tikzpicture}
  \end{center}
  where $\js_6'$ is the compression of the system with nesting tree
  
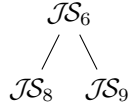
\begin{center}
    \begin{tikzpicture}[level distance=1cm,
    level 1/.style={sibling distance=1cm},
    level 2/.style={sibling distance=1cm}]
    \node {$\js_6$}
    child {node {$\js_8$}}
    child {node {$\js_9$}};
    \end{tikzpicture}
  \end{center}
  and $\js_7'$ is the compression of the system with nesting tree
  
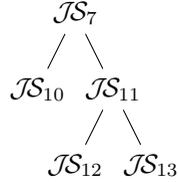
\begin{center}
    \begin{tikzpicture}[level distance=1cm,
    level 1/.style={sibling distance=1cm},
    level 2/.style={sibling distance=1cm}]
    \node {$\js_7$}
    child {node {$\js_{10}$}}
    child {node {$\js_{11}$}
      child {node {$\js_{12}$}}
      child {node {$\js_{13}$}}
    };
    \end{tikzpicture}
  \end{center}
\end{example}

Similarly to $\shrink$, applying $\expand$ iteratively, we can transform a justification in $\Compress(\js)$ to a justification in $\Merge(\js)$.
\begin{definition}
  Let $J$ be a justification in $\Compress(\js)$.
  By iteratively applying expand we get a justification in $\Merge(\js)$, which we will denote by $\expand(J)$.
\end{definition}

By \cref{cor:expandretainsvalue}, the value of $J$ and $\expand(J)$ are equal.
This solves the other side of the  equivalence between $\Merge(\js)$ and $\Compress(\js)$.
\begin{corollary}\label{cor:compresslessthanmerge}
  $\suppvalue_{\Compress(\js)}^t(x,\interp) \leqt \suppvalue_{\Merge(\js)}^t(x,\interp)$ for all defined facts $x$ and interpretations $\interp$.
\end{corollary}

\section{FO(FD)}\label{appendix:fo(fd)}

In this section, we will capture fixpoint definitions \cite{aspocp/HouD09,tplp/HouDD10} with nested justification theory.

We first show how first-order logic can be captured in justification theory.  Thereafter, we show how fixpoint definitions can be captured in  nested justification theory.

\subsection{FO system}

This section is not an application of nested systems, but is used as a subsystem in the other applications in the chapter.
So far, only propositional facts are used in justification theory, but it is not difficult to add first-order logic formulae.

We first recall the necessary preliminaries on first-order logic. A vocabulary $\Sigma$ consists of a set of predicate and function symbols. Terms and FO formulas are defined as usual, built up iteratively from variables, constants and functionm symbols, logical connectives ($\land$,$\lor$,$\lnot$) and quantifiers ($\forall,\exists$)
For a given domain $D$, a value for an $n$-ary predicate symbol is a function from $D^n$ to $\setl{\Fa, \Un, \Tr}$.
A $\Sigma$-interpretation $\interp$ consists of a domain $D^\interp$ and a value $\sigma^\interp$ for each symbol $\sigma$ in $\Sigma$.

If for example, we have a formula $\phi \land \psi$, then we add a rule $\phi \land \psi \gets \phi, \psi$, while for $\psi \lor \psi$ we add two rules $\phi \lor \psi \gets \phi$ and $\phi \lor \psi \gets \psi$.
This idea can be extended to any formula by looking at its decomposition into subformulae.
\begin{definition}
  Let $V$ be a set of first-order predicate symbols over a domain $D$ and let $\FOL(V)$ be the set of first-order formulae without free variables over $V$.
  To each formula $\phi$ in $\FOL(V)$, we associate a fact $\folfact{\phi}$.
  The justification frame $\jf_{\FOL(V),D}$ consists of $\Fp = \setprop{\folfact{\phi}}{\phi \in \FOL(V)}$, $\Fd \cap \Fp = \setprop{\folfact{\phi}}{\phi \in \FOL(V) \text{ and } \phi \text{ is not an atom}}$, and $\rules$ is the complementation of the inductive rule set defined as follows.
  \begin{itemize}
    \item $\folfact{\phi_1 \land \cdots \land \phi_n} \gets \folfact{\phi_1}, \ldots, \folfact{\phi_n}$;
    \item $\folfact{\phi_1 \lor \cdots \lor \phi_n} \gets \folfact{\phi_i}$ for all $i$;
    \item $\folfact{\neg \phi} \gets \tild \folfact{\phi}$;
    \item $\folfact{\forall x\colon \phi(x)} \gets \setprop{\folfact{\phi[x/d]}}{d \in D}$;
    \item $\folfact{\exists x\colon \phi(x)} \gets \folfact{\phi[x/d]}$ for all $d \in D$.
  \end{itemize}
\end{definition}
Note that we use $\neg$ to denote negation in first-order formulae to distinguish between $\tild$ in the justification frame.

This justification frame can capture first-order logic by letting the facts $\folfact{P(\overline{d})}$ and $\tild \folfact{P(\overline{d})}$ for $P \in V$ and $d \in D$ be the set of opens.
Usually, we identify $\folfact{P(\overline{d})}$ with $P(\overline{d})$.

\begin{example}\label{ex:fosystem}
  Suppose $cat$, $small$, and $cute$ are unary predicate symbols and we have a domain $ \setl{a, b}$
  Take the following first-order formula $\forall x\colon (cat(x) \land small(x)) \Rightarrow cute(x)$, where $\phi \Rightarrow \psi$ is just a shorthand for $\neg \phi \lor \psi$.
  A locally complete justification for this formula can look as follows.
  
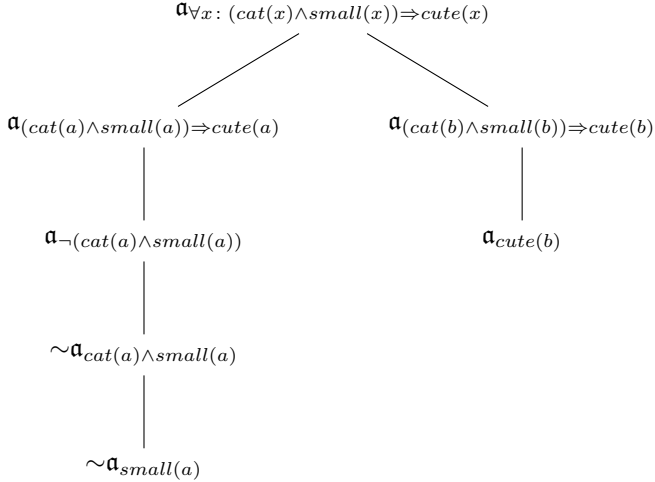
\begin{center}
    \begin{tikzpicture}[level distance=1.5cm,
    level 1/.style={sibling distance=5cm},
    level 2/.style={sibling distance=1.5cm}]
    \node {$\folfact{\forall x\colon (cat(x) \land small(x)) \Rightarrow cute(x)}$}
    child {node {$\folfact{(cat(a) \land small(a)) \Rightarrow cute(a)}$}
      child {node {$\folfact{\neg (cat(a) \land small(a))}$}
        child { node {$\tild \folfact{cat(a) \land small(a)}$}
          child {node {$\tild \folfact{small(a)}$}}
        }
      }
    }
    child {node {$\folfact{(cat(b) \land small(b)) \Rightarrow cute(b)}$}
      child {node {$\folfact{cute(b)}$}}
    };
    \end{tikzpicture}
  \end{center}
  Note that the fact $\tild \folfact{cat(a) \land small(a)}$ has two cases: $\setl{\tild \folfact{cat(a)}}$ and $\setl{\tild \folfact{small(a)}}$.
  This is due to complementation since $\folfact{cat(a) \land small(a)}$ has only one case: $\setl{\folfact{cat(a)}, \folfact{small(a)}}$.
  This justification shows that if $a$ is not small and $b$ is cute, then the formula $\forall x\colon (cat(x) \land small(x)) \Rightarrow cute(x)$ is true.
\end{example}

The idea in the example above works for any branch evaluation that maps finite branches to their last element.
This includes $\bekk$, $\bewf$, $\bewf$, and the alternative versions of $\besp$ and $\best$.
However, $\bekk$ is the better candidate in this case, since it says nothing about infinite branches ($\Un$) and it is the least precise such branch evaluation: $\interp(\bekk(\branch)) \leqp \interp(\be(\branch))$.

\begin{definition}\label{def:folsystem}
  Let $\js_{\FOL(V),D}$ be the justification system with justification frame $\jf_{\FOL(V),D}$ and branch evaluation $\bekk$.
\end{definition}

We can simulate first-order logic with justification theory, however note that justification theory is three-valued and thus will also be able to cope with unknown facts.
For instance in \cref{ex:fosystem}, if we do not know anything of $a$, but $b$ is a small cat, but not cute, then our formula is false.

\begin{proposition}
  For an interpretation of the atoms in $\FOL(V)$, the single model $\interp$ of $\js_{\FOL(V),D}$ extending this interpretation captures (three-valued) first-order logic: $\phi^I = \suppvalue_{\js}(\folfact{\phi}, \interp)$.
\end{proposition}
\begin{proof}
  Follows by induction on the depth of the formula.
\end{proof}

This system will be used extensively as a subsystem in a nested system and will essentially enable first-order logic formulae in justification systems.
Therefore, it is important to know that $\Flat(\js_{\FOL(V),D})$ is complementary.

We first need the following additional result, shown in \cite{phd/Marynissen22}:
\begin{proposition}[{Prop.\ 6.4.4. in \cite{phd/Marynissen22}}]\label{prop:wellfoundedrelationisflatcomplementary}
  Let $\js$ be a parametric and complementary justification system.
  There is a well-founded relation $\preceq$ on $\Fd$ so that
  \begin{itemize}
    \item if $x \gets A \in \rules$, then for all $y \in A$, $y \prec x$.
  \end{itemize}
  Then $\Flat(\js)$ is complementary.
\end{proposition}

\begin{proposition}\label{prop:flatfolsystemiscomplementary}
  The system $\Flat(\js_{\FOL(V),D})$ is complementary.
\end{proposition}
\begin{proof}
  Let $\phi \preceq \psi$ if $\phi$ is a subformula of $\psi$.
  This is a well-founded relation, then \cref{prop:wellfoundedrelationisflatcomplementary} finishes the proof.
\end{proof}

If you are only interested in a finite set of formulae, then $\js_{\FOL(V),D}$ can be restricted to the subformulae of these formulae.

\subsection{FO(FD), and its Representation in Nested Justification Theory}

A definitional rule over $\Sigma$ is an expression of the form: 
$$
\forall \overline{x} (P(\overline{x}))\gets  \phi
$$
where $P$ is a predicate symbol in $\Sigma$, $\overline{x}$ a tuple of $n$ variables, and $\phi$ an first-order
logic (FO) formula over $\Sigma$ such that the free variables of $\phi$ all occur in $\overline{x}$. Similar
to rules in justification frame, $P(\overline{x})$ is the head and $\phi$ the body of the rule.

\begin{definition}[Definition 1 in \cite{aspocp/HouD09}]
  A \emph{least fixpoint definition}, respectively \emph{greatest fixpoint definition} over a vocabulary $\Sigma$ is defined by simultaneous induction as an expression $\mathcal{D}$ of the form
  \begin{equation*}
  \left\lfloor \mathcal{R}, \Delta_1, \ldots, \Delta_m, \nabla_1, \ldots, \nabla_n \right\rfloor \quad \text{respectively}\quad \left\lceil \mathcal{R}, \Delta_1, \ldots, \Delta_m, \nabla_1, \ldots, \nabla_n \right\rceil
  \end{equation*}
  with $0 \leq n,m$ such that:
  \begin{enumerate}
    \item $\mathcal{R}$ is a set of definitional rules  over $\Sigma$.
    \item Each $\Delta_i$ is a least fixpoint definition over $\Sigma$ and each $\nabla_j$ is a greatest fixpoint definition over $\Sigma$.
    \item Every defined symbol in $\mathcal{D}$ has only positive occurrences in bodies of rules in $\mathcal{D}$.
    \item Each defined symbol $P \in \Def(\mathcal{D})$ has exactly one local definition, i.e. formally $\setl{\Def(\mathcal{R}), \Def(\Delta_1), \ldots, \Def(\nabla_n)}$ is a partition of $\Def(\mathcal{D})$.
    Formally: either locally defined or in some sub definitions.
    \item For every subdefinition $\mathcal{D'}$ of $\mathcal{D}$, $\Open(\mathcal{D}') \subseteq \Open(\mathcal{D}) \cup \Def(\mathcal{R})$.
  \end{enumerate}
  A \emph{fixpoint definition} is either a least or greatest fixpoint definition.
  The sets $\Def(\mathcal{D})$ and $\Open(\mathcal{D})$ are defined similarly as for inductive definitions.
\end{definition}

We assume that there is a single rule $\forall \overline{x} (P(\overline{x}) \gets \phi_P)$ for each $P \in \Def(\mathcal{D})$.

Given two disjoint first-order vocabularies $\Sigma$ and $\Sigma'$, and $\Sigma$-interpretation $I$ and $\Sigma'$-interpretation $I'$,
the $\Sigma \cup \Sigma'$-interpretation mapping each element $\sigma \in \Sigma$ to $\sigma^I$ and each $\sigma \in \Sigma'$ to $\sigma^{I'}$ is denoted by $I + I'$.
When $\Sigma' \subseteq \Sigma$, we denote the restriction of a $\Sigma$-interpretation $I$ to $\Sigma'$ by $\restrict{I}{\Sigma'}$.

Fix a domain $D$.
Let $\mathcal{R}$ be a set of definitional rules and $O$ a two-valued $\Open(\mathcal{R})$-interpretation.
We can associate an operator $\Gamma_O^\mathcal{R}$ on the set of $\Def(\mathcal{R})$-interpretations.

We define $\Gamma_O^\mathcal{R}(I_1) = I_2$ if for every $\forall \overline{x} (P(\overline{x}) \gets \phi_P) \in \mathcal{R}$, $P^{I_2} = \setprop{\overline{d} \in D^n}{\phi_P[\overline{x}/\overline{d}]^{I_1} = \Tr}$.

Since each defined symbol in $\Def(\mathcal{R})$ has only positive occurrences in the body of a rule in $\mathcal{R}$, we have that $\Gamma_O^\mathcal{R}$ is monotone with respect to $\leqt$.
Therefore, it has least and greatest fixpoints denoted by $\lfp(\Gamma_O^\mathcal{R})$ and $\gfp(\Gamma_O^\mathcal{R})$.

Let $\mathcal{D}$ be a fixpoint definition and let $O$ be a two-valued $\Open(\mathcal{D})$-interpretation.
We define an operator $\Gamma_O^\mathcal{D}$ on the set of $\Def(\mathcal{D})$-interpretations.

$\Gamma_O^\mathcal{D}(I)$ is defined inductively as the interpretation $K + K'$ where
\begin{itemize}
  \item $K$ is the $(\Def(\mathcal{D}) \setminus \Def(\mathcal{R}))$-interpretation such that for $I' = O + \restrict{I}{\Def(\mathcal{R})}$:
  \begin{itemize}
    \item $\restrict{K}{\Def(\Delta_i)} = \lfp(\Gamma_{I'}^{\Delta_i})$ for all $i = 1, \ldots, m$.
    \item $\restrict{K}{\Def(\nabla_j)} = \gfp(\Gamma_{I'}^{\nabla_j})$ for all $j=1,\ldots, n$.
  \end{itemize}
  \item $K'$ is $\Def(\mathcal{R})$-interpretation $\Gamma_{O + K}^\mathcal{R}(\restrict{I}{\Def(\mathcal{R})})$.
\end{itemize}

\begin{definition}
  Let $\mathcal{D}$ be a fixpoint definition and $I$ a two-valued interpretation of the symbols in $\mathcal{D}$.
  The interpretation $I$ is a \emph{model} of $\mathcal{D}$ if either
  \begin{itemize}
    \item $\mathcal{D}$ is a least fixpoint and $\restrict{I}{\Def(\mathcal{D})} = \lfp\left(\Gamma_{\restrict{I}{\Open(\mathcal{D})}}^\mathcal{D}\right)$.
    \item $\mathcal{D}$ is a greatest fixpoint and $\restrict{I}{\Def(\mathcal{D})} = \gfp\left(\Gamma_{\restrict{I}{\Open(\mathcal{D})}}^\mathcal{D}\right)$.
  \end{itemize}
\end{definition}
The similarities between fixpoint definitions and nested justification systems are abundant.
Fixpoint definitions have an additional restriction: every defined symbol has only positive occurrences in bodies of rules. On the level of justification frames, this restriction gives rise to additional properties:
\begin{definition}
  A justification frame $\jf=\jfcomplete$ is \emph{positive} if for every $x \in \Fp$ and $x \gets A \in \rules$ we have that $A \cap \Fd \cap \Fn = \emptyset$.
\end{definition}
In case a justification frame is positive, then the unique $\bewf$-model corresponds to the least fixpoint of $\opjf$ and similarly for $\becwf$.
\begin{proposition}\label{prop:positivesystemwfandcwfcharacterisation}
  Let $\jf=\jfcomplete$ be a positive complementary justification frame.
  If the interpretation of the opens is fixed, then
  \begin{itemize}
    \item the unique $\bewf$-model $\interp$ of $\jf$ corresponds (as in \cite{ijcai/MarynissenBD21}) to $(I,I)$ with $I=\lfp(\opjf)$.
    \item the unique $\becwf$-model $\interp$ of $\jf$ corresponds to $(I,I)$ with $I=\gfp(\opjf)$.
  \end{itemize}
\end{proposition}
\begin{proof}
  Since $\jf$ is positive, we have that $\opjf$ is monotone with respect to $\subseteq$, hence $\opjf$ has least and greatest fixpoints with respect to $\subseteq$.
  This also means that there is a $\leqt$-least $\besp$-model and $\leqt$-greatest $\besp$-models and they correspond with $\lfp(\opjf)$ and $\gfp(\opjf)$.
  Since $\jf$ is positive, the $\leqt$-least $\besp$-model is the unique $\best$-model.
  Since there is a single $\best$-model, it is equal to the unique $\bewf$-model.
  Similarly, the $\leqt$-greatest $\besp$-model is the unique $\becst$-model, which is equal to the unique $\becwf$-model.
\end{proof}

We now are going to construct a nested system corresponding to a fixpoint definition.
Take a fixpoint definition $\mathcal{D}$.
\begin{definition}
  For a set of definitional rules, the justification frame $\jf_\mathcal{R}$ contains the complementation of the rules
  \begin{equation*}
  \setprop{P(\overline{d}) \gets \folfact{\phi_P[\overline{x}/\overline{d}]}}{\forall \overline{x} (P(\overline{x} \gets \phi_P) \in \mathcal{R}}.
  \end{equation*}

  The nested system $\js_\mathcal{D}$ is defined inductively as the nested system with nesting tree (the leaves represent the nesting trees of those systems)
  \begin{center}
    \begin{tikzpicture}[level distance=1cm,
    level 1/.style={sibling distance=1cm},
    level 2/.style={sibling distance=1cm}]
    \node {$\js_\mathcal{R}$}
    child {node {$\js_{\FOL}$}
      child {node {$\js_{\Delta_1}$}}
      child {node {$\cdots$}}
      child {node {$\js_{\Delta_m}$}}
      child {node {$\js_{\nabla_1}$}}
      child {node {$\cdots$}}
      child {node {$\js_{\nabla_n}$}}
    };
    \end{tikzpicture}
  \end{center}
  where $\js_\mathcal{R}$ is $\jf_\mathcal{R}$ together with $\bewf$ if $\mathcal{D}$ is a least fixpoint and $\becwf$ is $\mathcal{D}$ is a greatest fixpoint.
  The system $\js_{\FOL}$ is the restriction of the system from \cref{def:folsystem} to the formula $\phi_P$ in $\mathcal{R}$.
\end{definition}

There is a slight hiccup in the definition above: it is possible that $\phi_P$ and $\phi_Q$ have subformulae in common.
In that case, some facts are defined in multiple levels.
We can avoid this by doubling the facts: every $\js_{\FOL}$ system uses a different symbol, for example $\folfact{\phi}$ and $\folfactfull{b}{\phi}$.
\begin{example}
  Let $T$ be a binary predicate symbol denoting a transition graph and $R$ a unary predicate on the states.
  The set of states that have a path passing infinitely many times through a state satisfying $R$ is captured by the predicate $P$ defined using the following fixpoint definition.
  \begin{equation*}
  \greatestframe{
    \forall x (P(x) \gets Q(x)) \\
    \leastframe{
      \forall x (Q(x) \gets R(x) \wedge \exists y\colon T(x,y) \wedge P(y)) \\
      \forall x (Q(x) \gets \exists y \colon T(x,y) \wedge Q(y))
    }
  }
  \end{equation*}
  The corresponding justification system is as follows.
  \begin{equation*}
  \becwf\colon \jframe{
    P(x) \gets Q(x) \\
    \bekk\colon \jframe{
      \bewf\colon \jframe{
        Q(x) \gets \folfact{\phi_Q} \\
        \bekk\colon \jframe{
          \folfact{\phi_Q} \gets \folfact{R(x) \wedge \exists y\colon T(x,y) \wedge P(y)} \\
          \folfact{\phi_Q} \gets \folfact{\exists y \colon T(x,y) \wedge Q(y)} \\
          \folfact{R(x) \wedge \exists y\colon T(x,y) \wedge P(y)} \gets R(x), \folfact{\exists y\colon T(x,y) \wedge P(y)} \\
          \folfact{\exists y\colon T(x,y) \wedge P(y)} \gets \folfact{T(x,y) \wedge P(y)} \\
          \folfact{\exists y\colon T(x,y) \wedge Q(y)} \gets \folfact{T(x,y) \wedge Q(y)} \\
          \folfact{T(x,y) \wedge P(y)} \gets T(x,y), P(y) \\
          \folfact{T(x,y) \wedge Q(y)} \gets T(x,y), Q(y)
        }
      }
    }
  }
  \end{equation*}
\end{example}

We first need the following additional results shown in  \cite{phd/Marynissen22}:
\begin{proposition}[{Prop.\ 6.4.13. in \cite{phd/Marynissen22}}]\label{prop:compresscomplementaryifsimplerules}
   Let $\js$ be a nested system of depth $2$ such that
   \begin{enumerate}
     \item all systems in its nesting tree are complementary;
     \item the flattening of all leaf systems are complementary;
     \item each defined fact $x$ in the top level has a single rule.
   \end{enumerate}
  Then $\Compress(\js)$ is complementary.
\end{proposition}
%
\begin{proposition}[{Prop.\ 6.4.5. in \cite{phd/Marynissen22}}]\label{prop:treeflatiscomplementary}
  If $\js$ is complementary and $\be$ respects negation, then $\Flat_t(\js)$ is complementary.
\end{proposition}
\begin{lemma}
  Let $\mathcal{D}$ be a fixpoint definition.
  Then $\Compress_t(\js_\mathcal{D})$ is complementary.
\end{lemma}
\begin{proof}
  Follows by induction using \cref{prop:compresscomplementaryifsimplerules,prop:flatfolsystemiscomplementary,prop:treeflatiscomplementary}.
\end{proof}

\begin{theorem}
  Let $\mathcal{D}$ be a fixpoint definition.
  Then the model of $\mathcal{D}$ is equal to the model of $\Compress(\js_\mathcal{D})$
\end{theorem}

\end{appendix}

\end{document}